\documentclass[lettersize,journal]{IEEEtran}

\usepackage{amsmath,amsfonts}
\usepackage{algorithmic}
\usepackage{algorithm}
\usepackage{array}
\usepackage[caption=false,font=normalsize,labelfont=sf,textfont=sf]{subfig}
\usepackage{textcomp}
\usepackage{stfloats}
\usepackage{url}
\usepackage{verbatim}
\usepackage{graphicx}
\usepackage{cite}
\usepackage{hyperref}
\usepackage{diagbox} 
\usepackage{bm}                   
\usepackage{amsthm}

\graphicspath{{figs/}}

\newtheorem{theorem}{Theorem}[section]
\newtheorem{proposition}[theorem]{Proposition}
\newtheorem{lemma}[theorem]{Lemma}

\newtheorem{definition}[theorem]{Definition}

\newtheorem{remark}[theorem]{Remark}

\DeclareMathOperator{\KL}{KL}  
\newcommand{\vect}[1]{\bm{#1}}

\newcommand{\N}[1]{\mathcal{#1}}

\begin{document}

\title{Relaxed Triangle Inequality for Kullback-Leibler Divergence Between Multivariate Gaussian Distributions}

\author{
	Shiji Xiao,\thanks{Shiji Xiao, Yufeng Zhang, Chubo Liu, Yan Ding and Kenli Li are with the College of Computer Science and Electronic Engineering, Hunan University, Changsha 410082, China (e-mail: \{xsj, yufengzhang, liuchubo, ding, lkl\}@hnu.edu.cn). Yufeng Zhang and Kenli Li are the corresponding authors.}
	\and
	Yufeng Zhang,
	\and
	Chubo Liu,
	\and
	Yan Ding,
	\and
	Keqin Li \textsl{Fellow IEEE} \thanks{Keqin Li is with the Department of Computer Science, State University of New York at New Paltz, 1 Hawk Drive, New Paltz, NY 12561, USA (e-mail: lik@newpaltz.edu).},
	\and
	Kenli Li
}




\maketitle

\begin{abstract}
	The Kullback-Leibler (KL) divergence is not a proper distance metric and does not satisfy the triangle inequality, posing theoretical challenges in certain practical applications.
	Existing work has demonstrated that KL divergence between multivariate Gaussian distributions follows a relaxed triangle inequality. Given any three multivariate Gaussian distributions $\mathcal{N}_1, \mathcal{N}_2$, and $\mathcal{N}_3$, if $\KL(\mathcal{N}_1, \mathcal{N}_2)\leq \epsilon_1$ and $\KL(\mathcal{N}_2, \mathcal{N}_3)\leq \epsilon_2$, then $\KL(\mathcal{N}_1, \mathcal{N}_3)< 3\epsilon_1+3\epsilon_2+2\sqrt{\epsilon_1\epsilon_2}+o(\epsilon_1)+o(\epsilon_2)$. However, the supremum of $\KL(\mathcal{N}_1, \mathcal{N}_3)$ is still unknown.  
	In this paper, we investigate the relaxed triangle inequality for the KL divergence between multivariate Gaussian distributions and give the supremum of $\KL(\mathcal{N}_1, \mathcal{N}_3)$ as well as the conditions when the supremum can be attained.
	When $\epsilon_1$ and $\epsilon_2$ are small, the supremum is $\epsilon_1+\epsilon_2+2\sqrt{\epsilon_1\epsilon_2}+o(\epsilon_1)+o(\epsilon_2)$. 
	Finally, we demonstrate several applications of our results in out-of-distribution detection with flow-based generative models and safe reinforcement learning.
\end{abstract}

 
\section{Introduction} \label{section_introduction}
\IEEEPARstart{T}{he} Kullback-Leibler divergence serves as a cornerstone in information theory and has found broad applications in machine learning and deep learning \cite{PRML, rioux2021BlindDeblurringBarcodes, wang2012RobustActiveStereo}, such as pattern recognition~\cite{lifang2017FeatureSelectionAlgorithm}, variational inference~\cite{blei2017VariationalInferenceReview, jordan1999IntroductionVariationalMethods}, generative modeling~\cite{kingma2022AutoEncodingVariationalBayes, rezende2015VariationalInferenceNormalizinga} and reinforcement learning~\cite{schulman2015TrustRegionPolicya, schulman2017ProximalPolicyOptimization}.
The KL divergence between two probability density functions $p(\vect{x})$ and $q(\vect{x})$ defined on $\mathbb{R}^n$ is 
\begin{equation*}
	\KL(p(\vect{x}) \,\|\, q(\vect{x})) = \int p(\vect{x}) \log \frac{p(\vect{x})}{q(\vect{x})} \, d\vect{x}
\end{equation*}
It is well-known that KL divergence is not a proper distance metric because it is not symmetric and does not satisfy the triangle inequality. This restricts the applications of KL divergence in many areas.

The Gaussian distribution is a fundamental modeling tool in  statistics~\cite{casella2024StatisticalInference,  fisher1922MathematicalFoundationsTheoretical}, information theory\cite{laparra2025EstimatingInformationTheoretic}, and machine learning~\cite{kingma2022AutoEncodingVariationalBayes, rezende2015VariationalInferenceNormalizing, zhang2024KullbackLeiblerDivergenceBasedOutofDistribution, yang2023DetectingRotatedObjects}.
Given two $n$-dimensional Gaussian distributions $\mathcal{N}_1 = \mathcal{N}(\bm{\mu}_1, \bm{\Sigma}_1)$ and $\mathcal{N}_2 = \mathcal{N}(\bm{\mu}_2, \bm{\Sigma}_2)$, the KL divergence between $\mathcal{N}_1$ and $\mathcal{N}_2$ has the following closed form ~\cite{pardo2018StatisticalInferenceBased}
\begin{equation} \label{eq_closedFormKLForGauss}
	\begin{aligned}
		\KL\left(\mathcal{N}_1 \parallel \mathcal{N}_2 \right) = \frac{1}{2} \Bigl( -\log{|\bm{\Sigma}_2^{-1}||\bm{\Sigma}_1|}	+ \operatorname{Tr}(\bm{\Sigma}_2 ^ {-1} \bm{\Sigma}_1) \\
		- n+ (\vect{\mu_2} - \vect{\mu_1})^\top \bm{\Sigma}_2 ^ {-1} (\vect{\mu_2} - \vect{\mu_1}) \Bigr)
  	\end{aligned}
\end{equation}
where the logarithm is taken to base $e$.
The closed form of KL divergence between Gaussian distributions allows further investigation of its properties \cite{pardo2018StatisticalInferenceBased}. 

Recently, Zhang \textit{et al.} reveal that KL divergence between multivariate Gaussian distributions follows a relaxed triangle inequality \cite{zhang2023PropertiesKullbackleiblerDivergence}. Given any three $n$-dimensional Gaussian distributions $\mathcal{N}_1$, $\mathcal{N}_2$, and $\mathcal{N}_3$ satisfying $\KL(\mathcal{N}_1 \,\|\, \mathcal{N}_2) \leq \epsilon_1$ and $\KL(\mathcal{N}_2 \,\|\, \mathcal{N}_3) \leq \epsilon_2$ with $\epsilon_1, \epsilon_2 \geq 0$, it holds that
\begin{equation}
\begin{aligned}\label{eq:2023relaxed}
	\KL\big(\mathcal{N}_1 \,\|\, \mathcal{N}_3\big)
	< \frac{1}{2} \Bigg[ \left( w_2(2\epsilon_1) - 1\right)    \left( w_2(2\epsilon_2) - 1\right) \\
	+ w_2(2\epsilon_2) \left( \sqrt{2\epsilon_1} + \sqrt{\frac{2\epsilon_2}{w_1(2\epsilon_2)}} \right)^2  \Bigg] + \epsilon_1 + \epsilon_2
\end{aligned}
\end{equation}
where $w_1(x) = -W_{0}\!\big(-e^{-(1+x)}\big)$ and $w_2(x) = -W_{-1}\!\big(-e^{-(1+x)}\big)$, and $W_{0}(x)$ and $W_{-1}(x)$ denote the principal and lower real branches of the Lambert $W$ function \cite{corless1996LambertWFunction}, respectively. 
When $\epsilon_1$ and $\epsilon_2$ are small, the inequality \eqref{eq:2023relaxed} becomes 
$$\KL\big(\mathcal{N}_1 \,\|\, \mathcal{N}_3\big)<3\epsilon_1+3\epsilon_2+2\sqrt{\epsilon_1\epsilon_2}+o(\epsilon_1)+o(\epsilon_2)$$
The relaxed triangle inequality opens up opportunities for applying KL divergence in a broader range of settings. For example, it can be used to derive out-of-distribution detection algorithm with flow-based models \cite{zhang2024KullbackLeiblerDivergenceBasedOutofDistribution}, extend one-step safe guarantee to multiple steps in reinforcement learning \cite{liu2022ConstrainedVariationalPolicy}, \textit{etc.}

In \cite{zhang2023PropertiesKullbackleiblerDivergence}, Zhang \textit{et al.} solve the following optimization problem to attain a relaxed triangle inequality.
\begin{equation} \tag{$P$} \label{problem_generalgauss}
	\begin{aligned}
		\text{max } & \KL\left(\N{N}(\bm{\mu}_1, \bm{\Sigma}_1) \, \| \, \N{N}(\bm{\mu}_3, \bm{\Sigma}_3)\right) \\
		\text{s.t.} & \KL\left(\N{N}(\bm{\mu}_1, \bm{\Sigma}_1) \, \| \, \N{N}(\bm{\mu}_2, \bm{\Sigma}_2)\right) = \Delta_1 \\
		& \KL\left(\N{N}(\bm{\mu}_2, \bm{\Sigma}_2) \, \| \, \N{N}(\bm{\mu}_2, \bm{\Sigma}_2)\right) = \Delta_2 
	\end{aligned}
\end{equation}

Leveraging the closed form of KL divergence between Gaussians and the properties of the Lambert \textit{W} function, Zhang \textit{et al.} developed an optimization-based proof and find the upper bound (Inequality \eqref{eq:2023relaxed}). However, the bound found in previous work \cite{zhang2023PropertiesKullbackleiblerDivergence} is not a strict one because Zhang \textit{et al.} relaxed the constraints in the proof for convenience. 
 
In this work, we step further in exploring the relaxed triangle inequality of KL divergence between multivariate Gaussian distributions. We aim to answer the following research question: \textit{suppose $\KL\left(\mathcal{N}_1 \,\|\, \mathcal{N}_2 \right) = \Delta_1$ and $\KL\left(\mathcal{N}_2 \,\|\, \mathcal{N}_3\right) = \Delta_2$, where $\Delta_1, \Delta_2 > 0$ are fixed finite constants, what is the supremum of $KL(\mathcal{N}_1||\mathcal{N}_3)$}? 	

In this work, we decompose the Problem \ref{problem_generalgauss} into two  optimization problems $P_{\vect{\mu}}$ and $P_{\bm{\Sigma}}$ and solve them individually. For the first problem $P_{\vect{\mu}}$, we use Cauchy-Schwarz Inequality to find its supremum. For the second problem $P_{\bm{\Sigma}}$, which has been solved  in previous work \cite{zhang2023PropertiesKullbackleiblerDivergence} with the support of several lemmas, we provide a more concise proof for one key lemma. Furthermore, we show that $P_{\bm{\mu}}$ and $P_{\bm{\Sigma}}$ share identical conditions for achieving their supremum. As a result, the supremum of  Problem \ref{problem_generalgauss} can be attained from the solutions of $P_{\vect{\mu}}$ and $P_{\bm{\Sigma}}$. 

The contributions of this work are as follows.
\begin{enumerate}
	\item We find that the dimension-free supremum of $\KL(\N{N}_1 \| \N{N}_3)$ is $\frac{1}{2}\left[ w_2(2 \Delta_1) - 1 \right]\left[ w_2(2 \Delta_2) - 1 \right] + \Delta_1 + \Delta_2$ if $\KL(\N{N}_1 \| \N{N}_2) = \Delta_1$ and $\KL(\N{N}_2 \| \N{N}_3) = \Delta_2$. We explicitly characterize the necessary and sufficient conditions for the supremum to be attained.
	
	\item The supremum of $\KL(\N{N}_1 \| \N{N}_3)$ is $\epsilon_1 + \epsilon_2  + 2 \sqrt{\epsilon_1 \epsilon_2} + o(\epsilon_1) + o(\epsilon_2)$ if $\KL(\N{N}_1 \| \N{N}_2) = \epsilon_1$ and $\KL(\N{N}_2 \| \N{N}_3) = \epsilon_2$ for small $\epsilon_1$ and $\epsilon_2$.
	
	\item The theoretical results developed in this paper can help the application of KL divergence in various domains including out-of-distribution detection with flow-based models and safe reinforcement learning.
\end{enumerate}

The remainder of this paper is organized as follows. Section~\ref{section_relatedwork} reviews related work. Section~\ref{section_notations} introduces the notation used throughout the paper. Section~\ref{section_KLsupremum} presents the core dimension-free supremum theorem concerning KL divergences between multivariate Gaussian distributions. In Section~\ref{section_discussion}, we compare our results with existing ones and discuss their applications. Finally, Section~\ref{section_conclusion} concludes the paper. Lengthy proofs are deferred to the Appendix.

\section{Related Work} \label{section_relatedwork}
\noindent
Since its introduction by Kullback and Leibler~\cite{kullback1951InformationSufficiency}, the KL divergence has played a pivotal role across information theory~\cite{csiszar1975IDivergenceGeometryProbability}, statistical inference~\cite{gkelsinis2022StatisticalInferenceBased, pardo2018StatisticalInferenceBased}, and machine learning~\cite{blei2017VariationalInferenceReview, schulman2015TrustRegionPolicy}.
To date, researchers have extensively investigated the KL divergence between various families of complex probability distributions, including Gaussian Mixture Models (GMMs)~\cite{durrieu2012LowerUpperBounds, hershey2007ApproximatingKullbackLeibler}, multivariate generalized Gaussian distributions~\cite{bouhlel2019KullbackLeiblerDivergence}, and Markov sources~\cite{rached2004KullbackLeiblerDivergenceRate}. 

Gaussian distributions has been  employed as a fundamental building block in generative models such as Variational Autoencoders~\cite{kingma2022AutoEncodingVariationalBayes} and normalizing flows~\cite{rezende2015VariationalInferenceNormalizinga}, as well as in a wide range of probabilistic modeling frameworks.
The KL divergence between multivariate Gaussian distributions admits a closed-form expression~\cite{pardo2018StatisticalInferenceBased} as shown in~\eqref{eq_closedFormKLForGauss}. 

KL divergence is not a proper distance metric because it is not symmetric and does not satisfy triangle inequality.
Recently, Zhang \textit{et al.}  systematically investigate how severely the KL divergence between multivariate Gaussian distributions deviates from satisfying symmetry and the triangle inequality \cite{zhang2023PropertiesKullbackleiblerDivergence}. 
For symmetry, they established infimum and supremum bounds. 
However, for the triangle inequality, given $\KL(\N{N}_1 \,\|\, \N{N}_2) \leq \varepsilon_1$ and $\KL(\N{N}_2 \,\|\, \N{N}_3) \leq \varepsilon_2$, they obtain a dimension-free but not a strict upper bound on $\KL(\N{N}_1 \,\|\, \N{N}_3)$.
Although this latter result is not tight, it nonetheless fills a key theoretical gap and provides a rigorous foundation for high-dimensional applications ~\cite{zhang2023PropertiesKullbackleiblerDivergence}.

To the best of our knowledge, no prior work has further characterized the tightness of this bound. In this paper, we derive the supremum of $\KL(\N{N}_1 \,\|\, \N{N}_3)$ under fixed values of $\KL(\N{N}_1 \,\|\, \N{N}_2)$ and $\KL(\N{N}_2 \,\|\, \N{N}_3)$, thereby establishing the optimal characterization of the deviation from the triangle inequality for KL divergence between Gaussian distributions.

\section{Notations} \label{section_notations}
\noindent
We begin by introducing the multi-valued Lambert $W$ function \cite{corless1996LambertWFunction}.
\begin{definition} \cite{corless1996LambertWFunction}
	The inverse function of $y = x e^x$ is called the Lambert $W$ function, denoted by $x = W(y)$.
\end{definition}

Following \cite{zhang2023PropertiesKullbackleiblerDivergence}, we adopt the functions $w_1(t)$ and $w_2(t)$ for $t \geq 0$, where $w_1(t)$ denotes the smaller solution and $w_2(t)$ denotes the larger solution of the equation $f(x) = x - \log x = 1 + t$. Their closed-form expressions are given in the following lemma.

\begin{lemma}\cite{zhang2023PropertiesKullbackleiblerDivergence}
	For $t \geq 0$, the equation $f(x) = 1 + t$ has two solutions. The smaller solution is given by
	\[
	w_1(t) = -W_{0}\!\big(-e^{-(1+t)}\big) \in (0, 1]
	\]
	and the larger solution is given by
	\[
	w_2(t) = -W_{-1}\!\big(-e^{-(1+t)}\big) \in [1, +\infty)
	\]
	where $W_{0}$ and $W_{-1}$ denote the principal and $-1$ branches of the Lambert $W$ function, respectively.
\end{lemma}
Based on $w_2(t)$, we then introduce three auxiliary functions $F(x, y)$, $G(x, y; \Delta_1, \Delta_2)$ and $H(x, y; \Delta_1, \Delta_2)$ defined on $\Omega(\Delta_1, \Delta_2)=[0, 2 \Delta_1] \times [0, 2 \Delta_2]$, which play a central role in our proofs:
\begin{equation*}
	\left\lbrace 
	\begin{aligned} 
        F(x, y) &= \left[ w_2(x) - 1 \right] \left[ w_2(y) - 1 \right] + x + y \\
		G(x, y; \Delta_1, \Delta_2) &= \left( \sqrt{w_2(y) \, (2\Delta_1 - x)} + \sqrt{2\Delta_2 - y} \right)^2 \\
		H(x, y; \Delta_1, \Delta_2) &= \frac{1}{2} \left[ F(x, y) + G(x, y; \Delta_1, \Delta_2) \right]
	\end{aligned}
	\right. 
\end{equation*}

Table \ref{tab_notations} summarizes  key notations used in this paper.

\begin{table*}[!t]
	\caption{Notations}
	\label{tab_notations}
	\centering
	\begin{tabular}{lc}
		\hline
		Notation & Description \\
		\hline
		$f(x)$ & The function $x - \log x$, with $x > 0$ \\
		$W(x)$ & The Lambert $W$ function \\
		$W_{-1}(x)$ & The $-1$ branch of $W(x)$ \\
		$w_1(t)$ & the smaller solution of $x - \log {x} = t + 1(t \geq 0)$, and satisfy $w_1(t) = -W_{0}\!\big(-e^{-(1+t)}\big)$, with $w_1(t) \leq 1 $ \\
		$w_2(t)$ & the larger solution of $x - \log {x} = t + 1(t \geq 0)$, and satisfy $w_2(t) = -W_{-1}\!\big(-e^{-(1+t)}\big)$, with $w_2(t) \geq 1 $ \\
		$\Omega(\Delta_1, \Delta_2) $ & $ [0, 2 \Delta_1] \times [0, 2 \Delta_2] $ \\
        $\operatorname{int}\left( \Omega(\Delta_1, \Delta_2)\right)$ & the interior of $\Omega(\Delta_1, \Delta_2)$, that is $(0,  2\Delta_1) \times (0, 2\Delta_2)$ \\
		$F(x,y)$ & $ \left[ w_2(x) - 1 \right] \left[ w_2(y) - 1 \right] + x + y$ \\
		$G(x,y; \Delta_1, \Delta_2)$ &  $\left( \sqrt{w_2(y) \, (2\Delta_1 - x)} + \sqrt{2\Delta_2 - y} \right)^2$, with $(x,y) \in \Omega(\Delta_1, \Delta_2)$ \\
  		$H(x, y; \Delta_1, \Delta_2)$ & $\frac{1}{2} \left[ F(x, y) + G(x, y; \Delta_1, \Delta_2) \right]$, with $(x,y) \in \Omega(\Delta_1, \Delta_2)$\\
		$\mathcal{N}(\mathbf{0}, \mathbf{I})$ & Standard normal distribution; the ambient dimension $n$ is omitted for brevity \\
		\hline
	\end{tabular}
\end{table*}

\section{Dimension-free Supremum} \label{section_KLsupremum}
\noindent
In this section, we establish a dimension-free supremum on $\KL\left(\N{N}_1 \| \N{N}_3 \right)$ given fixed values of $\KL\left(\N{N}_1 \| \N{N}_2 \right)$ and $\KL\left(\N{N}_2 \| \N{N}_3 \right)$ for any three Gaussian distributions.

\subsection{Core Lemma~\ref{lem_KLSupremumForNormal}} \label{subsection_KLSupremumForNormal}
\noindent
We first present Lemma~\ref{lem_KLSupremumForNormal}, which plays an essential role in establishing the subsequent theorem in this section.
\begin{lemma} \label{lem_KLSupremumForNormal}
	Let $\N{N}_i = \N{N}(\vect{\mu}_i, \bm{\Sigma}_i)$, $i \in \{1, 2\}$, be any two $n$-dimensional Gaussian distributions satisfying
	\begin{equation*}
		\left\{
		\begin{aligned}
			\KL\left(\mathcal{N}(\bm{\mu}_1, \bm{\Sigma}_1) \, \| \, \mathcal{N}(\vect{0}, \bm{I}) \right) &= \Delta_1 \\
			\KL\left( \mathcal{N}(\vect{0}, \bm{I}) \,\|\, \mathcal{N}(\bm{\mu}_2, \bm{\Sigma}_2) \right) &= \Delta_2
		\end{aligned}
		\right.
	\end{equation*}
	where $\Delta_1, \Delta_2 > 0$ are fixed constants, then
    \begin{equation*}
		\KL\bigl( \mathcal{N}_1 \,\|\, \mathcal{N}_2 \bigr) \leq \frac{1}{2} \left[ w_2(2 \Delta_1) - 1 \right]\left[ w_2(2 \Delta_2) - 1 \right] + \Delta_1 + \Delta_2
	\end{equation*}
	The equality holds if and only if
	\begin{equation*}
		\left\{
		\begin{aligned}
			\vect{\mu}_1 &= \vect{\mu}_2 = \vect{0} \\
			\bm{\Sigma}_1 &= \bm{Q} \operatorname{diag}\bigl(w_2(2\Delta_1), 1, \dots, 1\bigr) \bm{Q}^\top \\
			\bm{\Sigma}_2 &= \bm{Q} \operatorname{diag}\bigl(w_2(2\Delta_2)^{-1}, 1, \dots, 1\bigr) \bm{Q}^\top
		\end{aligned}
		\right.
	\end{equation*}
	where $\bm{Q}$ is an orthogonal matrix.
\end{lemma}

\begin{IEEEproof}
	Please see Appendix~\ref{appendix_KLSupremumForNormal} for details.
\end{IEEEproof}

The proof of Lemma~\ref{lem_KLSupremumForNormal} consists of the following seven steps.

\begin{enumerate}
	\item \textbf{Transforming to an Equivalent Problem.}
	Our problem is equivalent to Problem~\ref{problem_normalgauss}.
	\begin{equation} \tag{$P_{\N{N}}$} \label{problem_normalgauss}
		\begin{aligned}
			\text{max } & \KL\left(\N{N}(\bm{\mu}_1, \bm{\Sigma}_1) \, \| \, \N{N}(\bm{\mu}_2, \bm{\Sigma}_2)\right) \\
			\text{s.t.} & \KL\left(\N{N}(\bm{\mu}_1, \bm{\Sigma}_1) \, \| \, \N{N}(\vect{0}, \bm{I})\right) = \Delta_1 \\
			& \KL\left(\N{N}(\vect{0}, \bm{I}) \, \| \, \N{N}(\bm{\mu}_2, \bm{\Sigma}_2)\right) = \Delta_2 
		\end{aligned}
	\end{equation}
	By leveraging the close form of KL divergence between Gaussians and introducing intermediate variables $0 \leq x \leq 2 \Delta_1$ and $0 \leq y \leq 2 \Delta_2$, we eliminate the explicit $\KL$ expressions in Problem~\ref{problem_normalgauss} and reformulate it into the form presented in Problem~\ref{question_normalgaussTrans}.
	\begin{equation} \tag{$P_{\N{N}}^{'}$} \label{question_normalgaussTrans}
		\begin{aligned}
			\text{max } & \dfrac{1}{2} \Bigl(  -\log{|\bm{\Sigma}_2^{-1}||\bm{\Sigma}_1|} + \operatorname{Tr}(\bm{\Sigma}_2 ^ {-1} \bm{\Sigma}_1) - n \\
			& + (\vect{\mu_2} - \vect{\mu_1})^\top \bm{\Sigma}_2 ^ {-1} (\vect{\mu_2} - \vect{\mu_1}) \Bigr) \\
			\text{s.t. } & -\log|\bm{\Sigma}_1| + \operatorname{Tr}(\bm{\Sigma}_1) = n + x \\
			& -\log|\bm{\Sigma}_2 ^ {-1}| + \operatorname{Tr}(\bm{\Sigma}_2 ^ {-1} ) = n + y \\
			& \vect{\mu}_1^\top \vect{\mu}_1 = 2 \Delta_1 - x \\
			& \vect{\mu}_2^\top \bm{\Sigma}_2 ^ {-1} \vect{\mu}_2 = 2 \Delta_2 - y \\
			& (x,y) \in \Omega(\Delta_1, \Delta_2) \\
		\end{aligned}
	\end{equation}

	\item \textbf{Decomposing Problem \ref{question_normalgaussTrans}.}
	For any fixed $(x_{0}, y_{0}) \in \Omega(\Delta_1, \Delta_2)$, Problem~\ref{question_normalgaussTrans} can be decomposed into two subproblems coupled only through $\bm{\Sigma}_2$: Problem~\ref{problem_mu} (concerning both the means and covariance matrix $\bm{\Sigma_2}$) and Problem~\ref{problem_sigma} (concerning the covariance matrices).
	\begin{equation} \tag{$P_{\bm{\mu}}$} \label{problem_mu}  
		\begin{aligned}
			\text{max } & (\vect{\mu_2} - \vect{\mu_1})^\top \bm{\Sigma}_2 ^ {-1} (\vect{\mu_2} - \vect{\mu_1}) \\
			\text{s.t. } & -\log|\bm{\Sigma}_2 ^ {-1}| + \operatorname{Tr}(\bm{\Sigma}_2 ^ {-1} ) = n + y_{0} \\
			& \vect{\mu}_1^\top \vect{\mu}_1 = 2 \Delta_1 - x_{0} \\ 
			& \vect{\mu}_2^\top \bm{\Sigma}_2 ^ {-1} \vect{\mu}_2 = 2 \Delta_2 - y_{0} 
		\end{aligned}
	\end{equation}
    \begin{equation} \tag{$P_{\bm{\Sigma}}$} \label{problem_sigma}
		\begin{aligned}
			\text{max } & -\log{|\bm{\Sigma}_2^{-1}||\bm{\Sigma}_1|} + \operatorname{Tr}(\bm{\Sigma}_2 ^ {-1} \bm{\Sigma}_1) - n \\
			\text{s.t. } & -\log|\bm{\Sigma}_1| + \operatorname{Tr}(\bm{\Sigma}_1) = n + x_{0} \\
			& -\log|\bm{\Sigma}_2 ^ {-1}| + \operatorname{Tr}(\bm{\Sigma}_2 ^ {-1} ) = n + y_{0} 
		\end{aligned}
	\end{equation}
		
	\item \textbf{Solving Problem~\ref{problem_mu}.}
	We prove that the supremum of the objective function in Problem~\ref{problem_mu} equals $G(x_{0}, y_{0}; \Delta_1, \Delta_2)$ and derive sufficient conditions when the supremum can be attained.
	The solution to Problem~\ref{problem_mu} presented in Appendix~\ref{section_MaxMuFun} proceeds in the following three steps.
	\begin{enumerate} 
		\item We first fix $\vect{\Sigma}_2$, equivalently fixing the $n$ eigenvalues $\{ \lambda_i \}_{i = 1}^{n}$ and the corresponding unit eigenvectors $\{ \vect{e}_i \}_{i = 1}^{n}$ of $\vect{\Sigma}_2^{-1}$. Under this fixation, Problem~\ref{problem_mu} reduces to an optimization problem over $\vect{\mu}_1$ and $\vect{\mu}_2$.
		
		\item Next, we project $\vect{\mu}_1$ and $\vect{\mu}_2$ onto the eigenvector basis $\{ \vect{e}_i \}_{i = 1}^{n}$, thereby transforming the problem into an optimization over the vector components along these eigenvectors. By leveraging the equality constraint and applying the Cauchy--Schwarz inequality to bound the objective function, we establish that the conditional supremum of this component-wise optimization problem depends only on $\lambda_{\max} = \max_{i=1}^n \lambda_i$ and is strictly increasing in $\lambda_{\max}$ .
		
		\item Finally, we vary $\vect{\Sigma}_2$, \textit{i.e.}, adjust the eigenvalues $\{ \lambda_i \}_{i = 1}^{n}$, and optimize over these eigenvalues to obtain the global supremum $G(x_{0}, y_{0}; \Delta_1, \Delta_2)$.
	\end{enumerate}
	
	\item \textbf{Solving Problem~\ref{problem_sigma}.}
	Problem \ref{problem_sigma} has been solved in \cite[Appendix H]{zhang2023PropertiesKullbackleiblerDivergence}, in which the proof are supported by several key lemmas. In this paper, we provide a more concise proof for one of the key lemma.
    The supremum of the objective function in Problem~\ref{problem_sigma} equals $F(x_{0}, y_{0})$. The necessary and sufficient conditions when the supremum can be attained are also clarified. 
	Details are in Appendix~\ref{section_MaxSigmaFun}.

	\item \textbf{Checking Compatibility.}
	Since problem~\ref{problem_mu} and Problem~\ref{problem_sigma} couple through $\bm{\Sigma}_2$, we should check whether the conditions for attaining their respective suprema are simultaneously satisfiable. Using the necessary and sufficient condition from Lemma~\ref{lem_MaxSigmaFun} and sufficient conditions Lemma~\ref{lem_MaxMuFun}, we verify that for each fixed $(x_{0},y_{0})$, these requirements are compatible and thus the supremum of problem~\ref{question_normalgaussTrans} equals $H(x_{0},y_{0}; \Delta_1, \Delta_2)$.
	
	\item \textbf{Optimizing $H(x_{0},y_{0}; \Delta_1, \Delta_2)$.}
	We prove that $H(x_{0},y_{0}; \Delta_1, \Delta_2)$ with $(x_{0}, y_{0}) \in \Omega(\Delta_1, \Delta_2)$ attains its maximum $H^*(\Delta_1, \Delta_2) = \frac{1}{2} F(2\Delta_1, 2\Delta_2) = \frac{1}{2} \left[ w_2(2 \Delta_1) - 1 \right]\left[ w_2(2 \Delta_2) - 1 \right] + \Delta_1 + \Delta_2$ if and only if $(x_{0}, y_{0}) = (2\Delta_1, 2\Delta_2)$ as shown in Lemma~\ref{lem_MaxH}. 
	The proof of Lemma~\ref{lem_MaxH} is provided in Appendix~\ref{appendix_MaxH}, which consists of the following four steps. 
	\begin{enumerate}
		\item We first establish the existence of a point $(x^*, y^*) \in \Omega(\Delta_1, \Delta_2)$ at which $H(x, y; \Delta_1, \Delta_2)$ attains its maximum.
		
		\item Since $H(x, y; \Delta_1, \Delta_2)$ is continuously differentiable on $\operatorname{int}\left( \Omega(\Delta_1, \Delta_2)\right) = (0,  2\Delta_1) \times (0, 2\Delta_2)$, a necessary condition for a local maximum at an interior point $(x, y) \in \operatorname{int}\left( \Omega(\Delta_1, \Delta_2)\right)$ is
		\begin{equation}
			\left\lbrace 
			\begin{aligned}\label{eq:Hprime_zero}
				\frac{\partial H(x, y; \Delta_1, \Delta_2)}{\partial x} = 0 \\
				\frac{\partial H(x, y; \Delta_1, \Delta_2)}{\partial y} = 0
			\end{aligned}
			\right.
		\end{equation}
		By equivalently reformulating Equation \eqref{eq:Hprime_zero}, we derive a simplified equation and prove that it has no solution in $\operatorname{int}\left( \Omega(\Delta_1, \Delta_2)\right)$. Consequently, the maximizer $(x^*, y^*)$ cannot lie in the interior of $\Omega(\Delta_1, \Delta_2)$.
		
		\item We then analyze the monotonicity of $H(x, y; \Delta_1, \Delta_2)$ along each of the four edges of the boundary of $\Omega(\Delta_1, \Delta_2)$ and show that the maximum must be attained at either $(0, 0)$ or $(2\Delta_1, 2\Delta_2)$.
		
		\item Finally, a change of variables demonstrates that $H(2\Delta_1, 2\Delta_2; \Delta_1, \Delta_2) > H(0, 0; \Delta_1, \Delta_2)$, which yields the desired conclusion.
	\end{enumerate}

    \item \textbf{Establishing Necessary and Sufficient Conditions.} After fixing $(x_{0}, y_{0})=(2\Delta_1, 2\Delta_2)$ to maximize $H(x, y; \Delta_1, \Delta_2)$, 
    we show that the conditions required for Problems~\ref{problem_mu} and~\ref{problem_sigma} to simultaneously attain their suprema are both necessary and sufficient for Problem~\ref{question_normalgaussTrans} to attain its supremum. 
	Consequently, we obtain the supremum of Problem~\ref{problem_normalgauss} together with the necessary and sufficient conditions for its attainment. 
	
\end{enumerate}

We note that Lemma~\ref{lem_MaxH} constitutes one of the core contributions of this paper. A key technical trick in its proof is utilizing the following property: when maximizing a continuous function over a compact set, if the function is differentiable in the interior of the domain but cannot attain a critical point (\textit{i.e.}, a point where all partial derivatives are zero), then the supremum must be attained on the boundary. This observation reduces the original optimization problem to a lower-dimensional one that is more tractable.

Notably, the similar technique was also employed in the proof of Lemma~\ref{lem_MaxF}, a strengthened variant of the original Lemma~G.5 in \cite{zhang2023PropertiesKullbackleiblerDivergence}, leading to a more concise argument than the original. See Appendix~\ref{section_MaxSigmaFun} for details.

\subsection{Core Dimension-free Supremum} \label{subsection_KLSupremumForGeneral}
\noindent
For general parameter regimes, the result is summarized in Theorem~\ref{theorem_KLSupremumForGeneral}, which constitutes one of the core contributions of this work.
\begin{theorem} \label{theorem_KLSupremumForGeneral}
	Let $\N{N}_i = \N{N}(\vect{\mu}_i, \bm{\Sigma}_i)$, $i \in \{1, 2, 3\}$ be three $n$-dimensional Gaussian distributions satisfying
	\begin{equation*}
		\left\{
		\begin{aligned}
			\KL\left(\N{N}_1 \, \| \, \N{N}_2 \right) &= \Delta_1 \\
			\KL\left(\N{N}_2 \,\|\, \N{N}_3\right) &= \Delta_2
		\end{aligned}
		\right.
	\end{equation*}
	where $\Delta_1, \Delta_2 > 0$ are fixed constants, $\N{N}_2$ is a fixed Gaussian distribution, and $\N{N}_1$, $\N{N}_3$ are Gaussian distributions with variable parameters, then
	\begin{equation*}
		\KL\left(\N{N}_1 \,\|\, \N{N}_3 \right) \leq \frac{1}{2} \left[ w_2(2 \Delta_1) - 1 \right]\left[ w_2(2 \Delta_2) - 1 \right] + \Delta_1 + \Delta_2
	\end{equation*}
	The equality holds if and only if
	\begin{equation*}
		\left\{
		\begin{aligned}
			\bm{\mu}_1 &= \bm{\mu}_2 = \bm{\mu}_3\\
			\bm{\Sigma}_1 &= \bm{B}_2 \bm{Q} \operatorname{diag}\bigl(w_2(2 \Delta_1), 1, \dots, 1\bigr) \bm{Q}^\top \bm{B}_2^\top \\
			\bm{\Sigma}_3 &= \bm{B}_2 \bm{Q} \operatorname{diag}\bigl(w_2(2 \Delta_2)^{-1}, 1, \dots, 1\bigr) \bm{Q}^\top \bm{B}_2^\top
		\end{aligned}
		\right.
	\end{equation*}
	where $\bm{B}_2$ is an invertible matrix such that $\bm{\Sigma}_2 = \bm{B}_2 \bm{B}_2^{\top}$ and $\bm{Q}$ is an orthogonal matrix.
\end{theorem}

\begin{remark}
	Theorem~\ref{theorem_KLSupremumForGeneral} addresses the nontrivial case where $\Delta_1, \Delta_2 > 0$. When $\Delta_1 = 0$, we have $\bm{\mu}_1 = \bm{\mu}_2$ and $\bm{\Sigma}_1 = \bm{\Sigma}_2$, implying $\KL\left(\N{N}_1 \,\|\, \N{N}_3 \right) = \KL\left(\N{N}_2 \,\|\, \N{N}_3 \right) = \Delta_2$. Similarly, when $\Delta_2 = 0$, it follows that $\KL\left(\N{N}_1 \,\|\, \N{N}_3 \right) = \KL\left(\N{N}_1 \,\|\, \N{N}_2 \right) = \Delta_1$. These two cases can be unified as follows: if $\Delta_1\Delta_2 = 0$, then
	\[
	\KL\left(\N{N}_1 \,\|\, \N{N}_3 \right) = \Delta_1 + \Delta_2
	\]
	Thus, the boundary cases $\Delta_1\Delta_2 = 0$ are consistent with the formula in Theorem~\ref{theorem_KLSupremumForGeneral} and can be regarded as its natural extension.
\end{remark}

\begin{IEEEproof}
	Please see Appendix~\ref{section_KLSupremumForGeneral} for details.
\end{IEEEproof}

The proof of Theorem~\ref{theorem_KLSupremumForGeneral} mainly consists of the following three steps.
\begin{enumerate}
	\item \textbf{Applying an Invertible Linear Transformation.} \label{item_1}
	To solve Problem~\ref{problem_generalgauss}, we first apply the same invertible linear transformation to $\N{N}_1$, $\N{N}_2$, and $\N{N}_3$ simultaneous such that $\N{N}_2$ is transformed to the standard Gaussian $\N{N}(\vect{0}, \bm{I})$. 
	Then we can reformulate Problem~\ref{problem_generalgauss} into the form presented in Problem~\ref{problem_P'}.
	\begin{equation} \tag{$P^{'}$} \label{problem_P'}
		\begin{aligned}
			\text{max } & \KL \left( \N{N}_1^{'} \left( \bm{\mu}_1^{'}, \bm{\Sigma}_1^{'} \right)  \,\|\, \N{N}_3^{'} \left( \bm{\mu}_3^{'}, \bm{\Sigma}_3^{'} \right) \right) \\
			\text{s.t.} & \KL\left(\N{N}_1^{'} \left( \bm{\mu}_1^{'}, \bm{\Sigma}_1^{'} \right) \, \| \, \N{N}(\vect{0}, \bm{I})\right) = \Delta_1 \\
			& \KL\left(\N{N}(\vect{0}, \bm{I}) \, \| \, \N{N}_3^{'} \left( \bm{\mu}_3^{'}, \bm{\Sigma}_3^{'} \right) \right) = \Delta_2 
		\end{aligned}
	\end{equation}
	where 
	\begin{equation*} 
		\left\{\begin{aligned}
			\bm{\mu}_1^{'} &= \bm{B}_2^{-1} (\vect{\mu}_1 - \vect{\mu}_2) \\ 
			\bm{\Sigma}_1^{'} &= \bm{B}_2^{-1} \bm{\Sigma}_1 \bm{B}_2^\top \\ 
			\bm{\mu}_3^{'} &= \bm{B}_2^{-1} (\vect{\mu}_3 - \vect{\mu}_2) \\ 
			\bm{\Sigma}_3^{'} &= \bm{B}_2^{-1} \bm{\Sigma}_3 \bm{B}_2^\top \\ 
		\end{aligned} \right.
	\end{equation*}
	
	\item \textbf{Applying Lemma~\ref{lem_KLSupremumForNormal}.} \label{item_2}
	By applying Lemma~\ref{lem_KLSupremumForNormal}, we directly obtain the necessary and sufficient conditions that $\bm{\mu}_1^{'}, \bm{\Sigma}_1^{'}, \bm{\mu}_3^{'}, \bm{\Sigma}_3^{'}$ must satisfy for Problem~\ref{problem_P'} to attain its optimal solution.
	
	\item \textbf{Inverting the Linear Transformation.} \label{item_4}
	Solving the above system of equations yields the necessary and sufficient conditions that $\bm{\mu}_1, \bm{\Sigma}_1, \bm{\mu}_3, \bm{\Sigma}_3$ must satisfy.
	This step amounts to applying the inverse linear transformation simultaneously to $\N{N}_1^{'}$, the standard Gaussian $\N{N}(\vect{0}, \bm{I})$, and $\N{N}_3^{'}$, thereby mapping $\N{N}(\vect{0}, \bm{I})$ back to $\N{N}_2$. This transformation transports the optimal parameters and equality conditions from the normalized setting back to the original coordinate system. Details are provided in Appendix~\ref{section_KLSupremumForGeneral}.
\end{enumerate}

For small divergences regimes, the supremum in Theorem~\ref{theorem_KLSupremumForGeneral} can be rewritten as follows.

\begin{theorem} \label{theorem_KLUpperBoundForEpsilon}
	Let $\N{N}_i = \N{N}(\vect{\mu}_i, \bm{\Sigma}_i)$, $i \in \{1, 2, 3\}$, be three $n$-dimensional Gaussian distributions satisfying
	\begin{equation*}
		\left\{
		\begin{aligned}
			\KL\left(\N{N}_1 \, \| \, \N{N}_2 \right) &= \epsilon_1 \\
			\KL\left(\N{N}_2 \, \| \, \N{N}_3 \right) &= \epsilon_2
		\end{aligned}
		\right.
	\end{equation*}
	where $\epsilon_1, \epsilon_2 > 0$ are fixed small constants, $\N{N}_2$ is a fixed Gaussian distribution, and $\N{N}_1$, $\N{N}_3$ are Gaussian distributions with variable parameters. Then it holds that
	\begin{equation*}
		\KL\left(\N{N}_1 \, \| \, \N{N}_3 \right) \leq \epsilon_1 + \epsilon_2  + 2 \sqrt{\epsilon_1 \epsilon_2} + o(\epsilon_1) + o(\epsilon_2)
	\end{equation*}
\end{theorem}
\begin{IEEEproof}
	Please see Appendix~\ref{section_KLUpperBoundForEpsilon} for details.
\end{IEEEproof}

\subsection{Experiments}
\noindent
\subsubsection{The Properties of Supremum}
We begin with a numerical investigation of the supremum of $\KL\left(\N{N}_1 \,\|\, \N{N}_3 \right)$, which is expressed by the function 
\[
\frac{1}{2}\left[ w_2(2\Delta_1) - 1 \right] \left[ w_2(2\Delta_2) - 1 \right] + \Delta_1 + \Delta_2
\]
in Theorem~\ref{theorem_KLSupremumForGeneral}. Figure~\ref{fig_HStar} displays the heatmap and surface of this supremum. Table~\ref{table_Supremum} presents the numerical results of the supremum for various values of $\Delta_1$ and $\Delta_2$.
When $\Delta_1=\Delta_2=0.001$, the supremum is approximately $0.0041$, which is close to $4\Delta_1$.

As shown in Figure~\ref{fig_HStar}, the supremum increases monotonically with respect to both $\Delta_1$ and $\Delta_2$ and exhibits symmetry under the exchange of $\Delta_1$ and $\Delta_2$. 
Meanwhile, Table~\ref{table_Supremum} shows that when $\Delta_1 = \Delta_2$ and both are small, for instance less than $0.01$, the supremum is slightly greater than four times their common value, which is consistent with Theorem~\ref{theorem_KLSupremumForGeneral}.

\begin{figure*}[ht]
	\vskip 0.2in
	\begin{center}
		\centering
		\includegraphics[width=0.99 \linewidth]{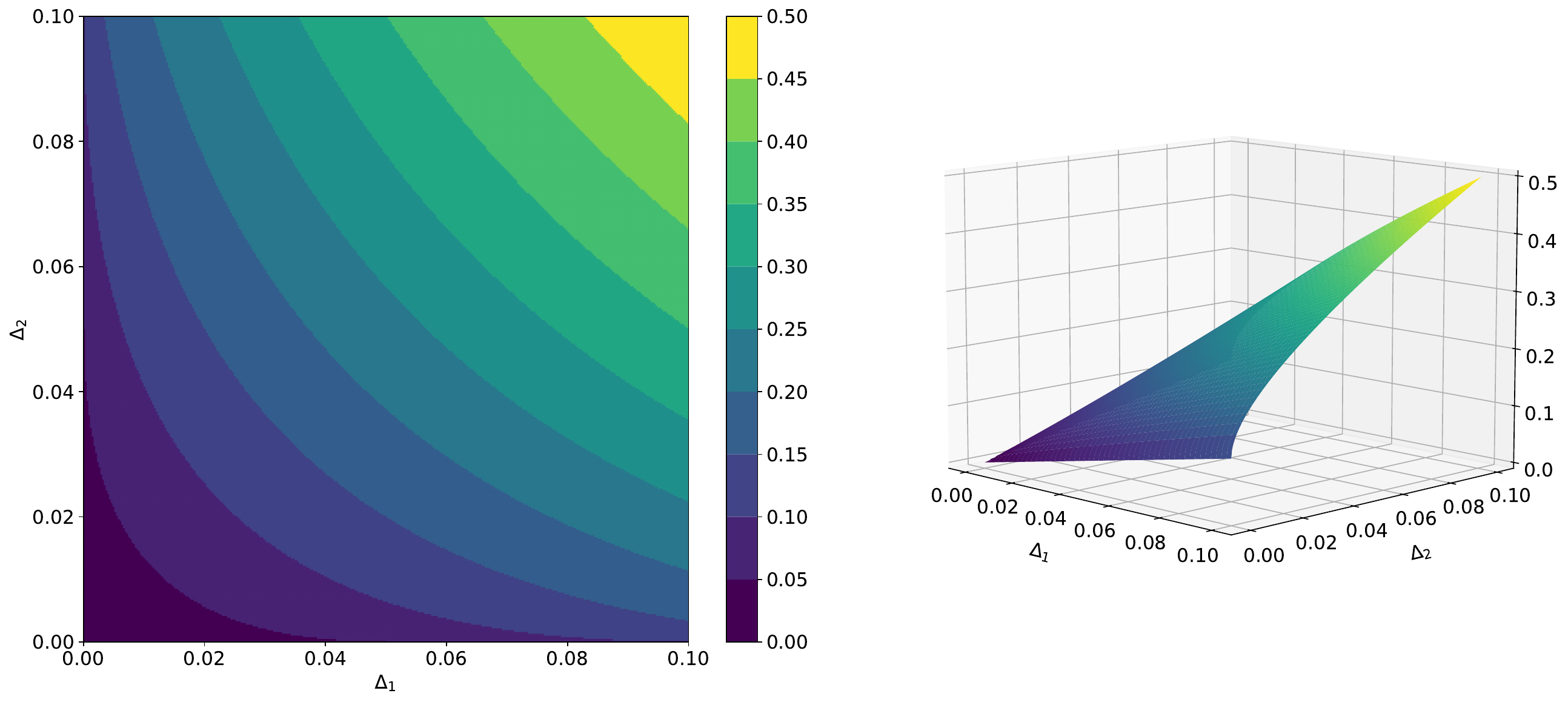}
		\caption{Heatmap and surface of the supremum of $\KL\left(\N{N}_1 \,\|\, \N{N}_3 \right)$, which is equal to $\frac{1}{2}\left[ w_2(2\Delta_1) - 1 \right] \left[ w_2(2\Delta_2) -  1 \right] + \Delta_1 + \Delta_2$.}
		\label{fig_HStar}
	\end{center}
\end{figure*}

\begin{table}[!t] 
	\begin{center}
		\caption{Supremum values for various values of $\Delta_1$ and $\Delta_2$.}
		\label{table_Supremum}
		\begin{tabular}{| c | c | c | c | c |}
			\hline
			\diagbox[width=2.5em, height=2.0em, innerleftsep=0.4em, innerrightsep=0pt]{$\Delta_1$}{$\Delta_2$} & 0.001 & 0.01 & 0.1 & 1  \\
			\hline
			0.001 & 0.0041 & 0.0179 & 0.1259 & 1.1142 \\
			\hline
			0.01 & 0.0179 & 0.0428 & 0.1925 & 1.3843 \\
			\hline
			0.1 & 0.1259 & 0.1925 & 0.4982 & 2.4535 \\
			\hline
			1 & 1.1142 & 1.3843 & 2.4535 & 8.1434 \\
			\hline
		\end{tabular}
	\end{center}
\end{table}

\subsubsection{The Condition for Supremum}
\textbf{One-dimensional Case.}
We conduct a numerical experiment to validate the condition for achieving the supremum in Theorem~\ref{theorem_KLSupremumForGeneral} in the one-dimensional case. Without loss of generality, we directly analyze $\KL\left(\N{N}_1 \, \| \, \N{N}_2  \right)$ for different $\N{N}_1$ and $\N{N}_2$ under the constraints
\begin{equation*}
	\left\{
	\begin{aligned}
		\KL\left(\N{N}({\mu}_1, {\sigma}_1^{2}) \, \| \, \mathcal{N}(0, 1) \right) &= \Delta_1 \\
		\KL\left( \mathcal{N}(0, 1) \,\|\, \N{N}({\mu}_2, {\sigma}_2^{2}) \right) &= \Delta_2
	\end{aligned}
	\right.
\end{equation*}
with $\Delta_1 = \Delta_2 = 0.1$.
These constraints can be equivalently rewritten as
\begin{equation*}
	\left\{
	\begin{aligned}
		- \log {\sigma}_1^{2} +  {\sigma}_1^{2} - 1 +  {\mu}_1^{2} &= 2 \Delta_1 \\
		- \log \dfrac{1}{{\sigma}_2^{2}} +  \dfrac{1}{{\sigma}_2^{2}} - 1 +  \dfrac{{\mu}_2^{2}}{{\sigma}_2^{2}}  &= 2 \Delta_2
	\end{aligned}
	\right.
\end{equation*}
Using the definitions of $w_1(x)$ and $w_2(x)$, one can verify that the parametric equation $k x - \log x = 1 + t$ for $x$ admits two solutions given by $x = \dfrac{w_1(t - \log k)}{k}$ or $x = \dfrac{w_2(t - \log k)}{k}$, where $k > 0$ and $t \geq 0$ are parameters.
Consequently, ${\sigma}_1^{2}$ and ${\sigma}_2^{2}$ are each determined by ${\mu}_1$ and ${\mu}_2$ up to two possible branches.
On the other hand, Theorem~\ref{theorem_KLSupremumForGeneral} states that $\KL\left(\N{N}_1 \, \| \, \N{N}_2  \right)$ attains its supremum when ${\sigma}_1^{2} = w_2(2 \Delta_1)$ and ${\sigma}_2^{2} = w_2(2 \Delta_2)^{-1}$.
Thus we focus exclusively on the branch associated with $w_2(x)$. This yields
\begin{equation} \label{eq_sigmaSolution}
	\left\{
	\begin{aligned}
		{\sigma}_1^{2} &= w_2\left( 2 \Delta_1 - {\mu}_1^{2} \right) \\
		\dfrac{1}{{\sigma}_2^{2}} &= \dfrac{ w_2\left[ 2 \Delta_2 - \log \left( {1 + {\mu}_2^{2}} \right) \right]  }{1 + {\mu}_2^{2}}
	\end{aligned}
	\right.
\end{equation}
where $\mu_1 \in \left[ -\sqrt{2\Delta_1}, \sqrt{2\Delta_1} \right]$ and ${\mu}_2 \in \left[ -\sqrt{e^{2\Delta_2} - 1}, \sqrt{e^{2\Delta_2} - 1} \right]$.
Under these conditions, for $i \in \{1, 2\}$, the distribution $\N{N}_i = \N{N}({\mu}_i, {\sigma}_i^{2})$ is uniquely determined by ${\mu}_i$. Figure~\ref{fig_N12_mu} displays the probability density functions of $\N{N}_1$ and $\N{N}_2$ for different values of their means under the imposed constraints.
\begin{figure}[ht]
	\vskip 0.2in
	\begin{center}
		\centering
		\includegraphics[width=0.99 \linewidth]{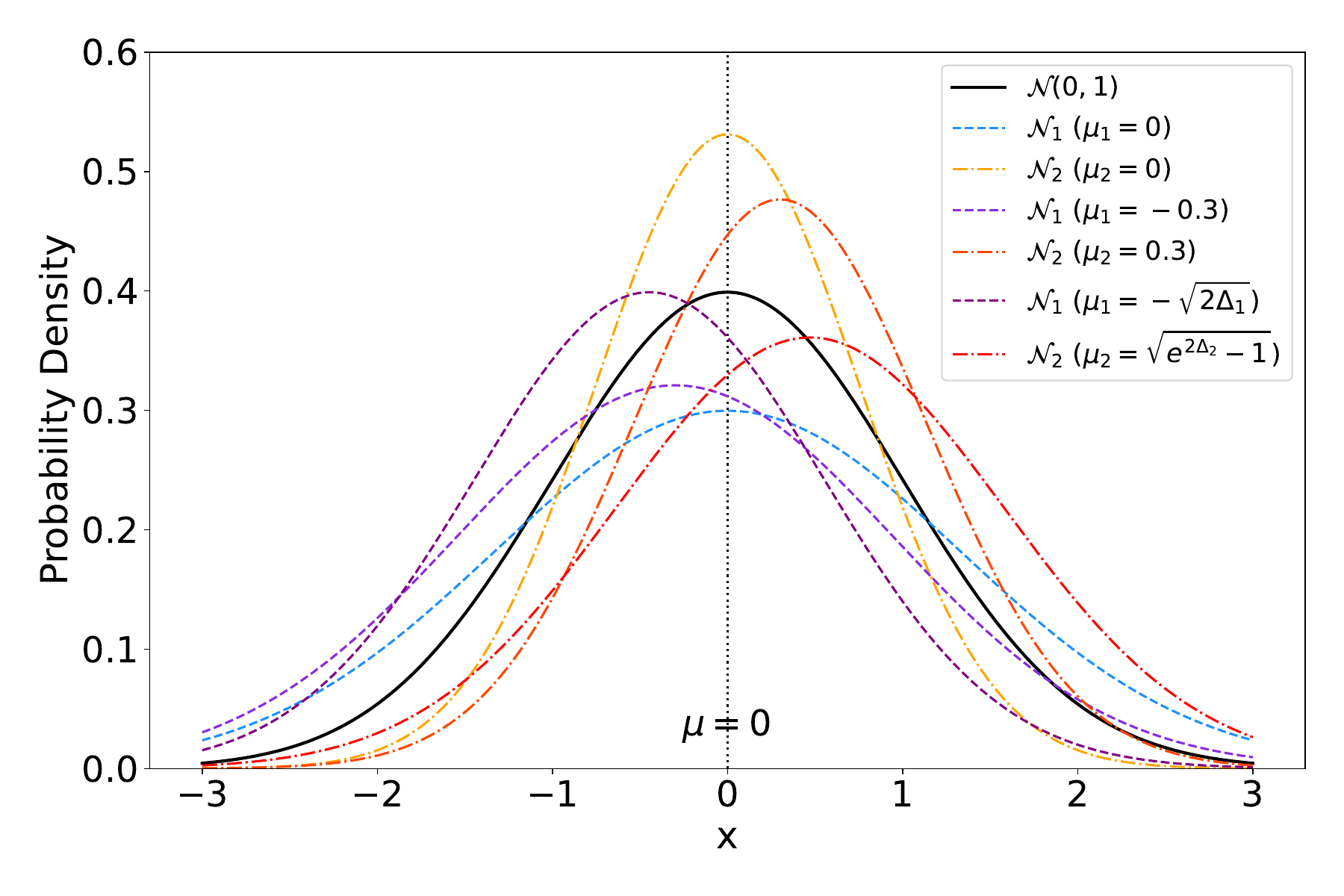}
		\caption{Probability density functions of $\N{N}_1({\mu}_1, {\sigma}_1^{2})$ and $\N{N}_2({\mu}_2, {\sigma}_2^{2})$ for varying means $\mu_1$ and $\mu_2$, respectively, where $\N{N}_1$ satisfies ${\sigma}_1^{2} = w_2\left( 2 \Delta_1 - {\mu}_1^{2} \right)$ and $\N{N}_2$ satisfies $\dfrac{1}{{\sigma}_2^{2}} = \dfrac{ w_2\left[ 2 \Delta_2 - \log \left( {1 + {\mu}_2^{2}} \right) \right]  }{1 + {\mu}_2^{2}}$, with $\Delta_1 = \Delta_2 = 0.1$.}
		\label{fig_N12_mu}
	\end{center}
\end{figure}

Accordingly, $\KL\left(\N{N}_1 \, \| \, \N{N}_2  \right)$ depends only on ${\mu}_1$ and ${\mu}_2$.
Substituting Equation \eqref{eq_sigmaSolution} into
\begin{equation*}
	\begin{aligned}
		\KL(\mathcal{N}_1 \| \mathcal{N}_2) = \dfrac{1}{2} \left( - \log \dfrac{{\sigma}_1^{2}}{{\sigma}_2^{2}} +  \dfrac{{\sigma}_1^{2}}{{\sigma}_2^{2}} - 1 +  \dfrac{ \left({\mu}_2 - {\mu}_1 \right) ^{2}}{{\sigma}_2^{2}} \right)
	\end{aligned}
\end{equation*}
yields an explicit expression for $\KL(\mathcal{N}_1 \| \mathcal{N}_2)$ as a function of ${\mu}_1$ and ${\mu}_2$.
The resulting heatmap is shown in Figure~\ref{fig_KL12_mu}.
As illustrated in Figure~\ref{fig_KL12_mu}, $\KL(\mathcal{N}_1 \| \mathcal{N}_2)$ attains its maximum at $({\mu}_1, {\mu}_2) = (0, 0)$, that is, when both means coincide with the mean of $\mathcal{N}(\vect{0}, \bm{I})$ and thus 
${\sigma}_1^{2} = w_2(2 \Delta_1)$,  ${\sigma}_2^{2} = w_2(2 \Delta_2)^{-1}$. This observation confirms all the conditions stated in Theorem~\ref{theorem_KLSupremumForGeneral}.

\begin{figure}[ht]
	\vskip 0.2in
	\begin{center}
		\centering
		\includegraphics[width=0.99 \linewidth]{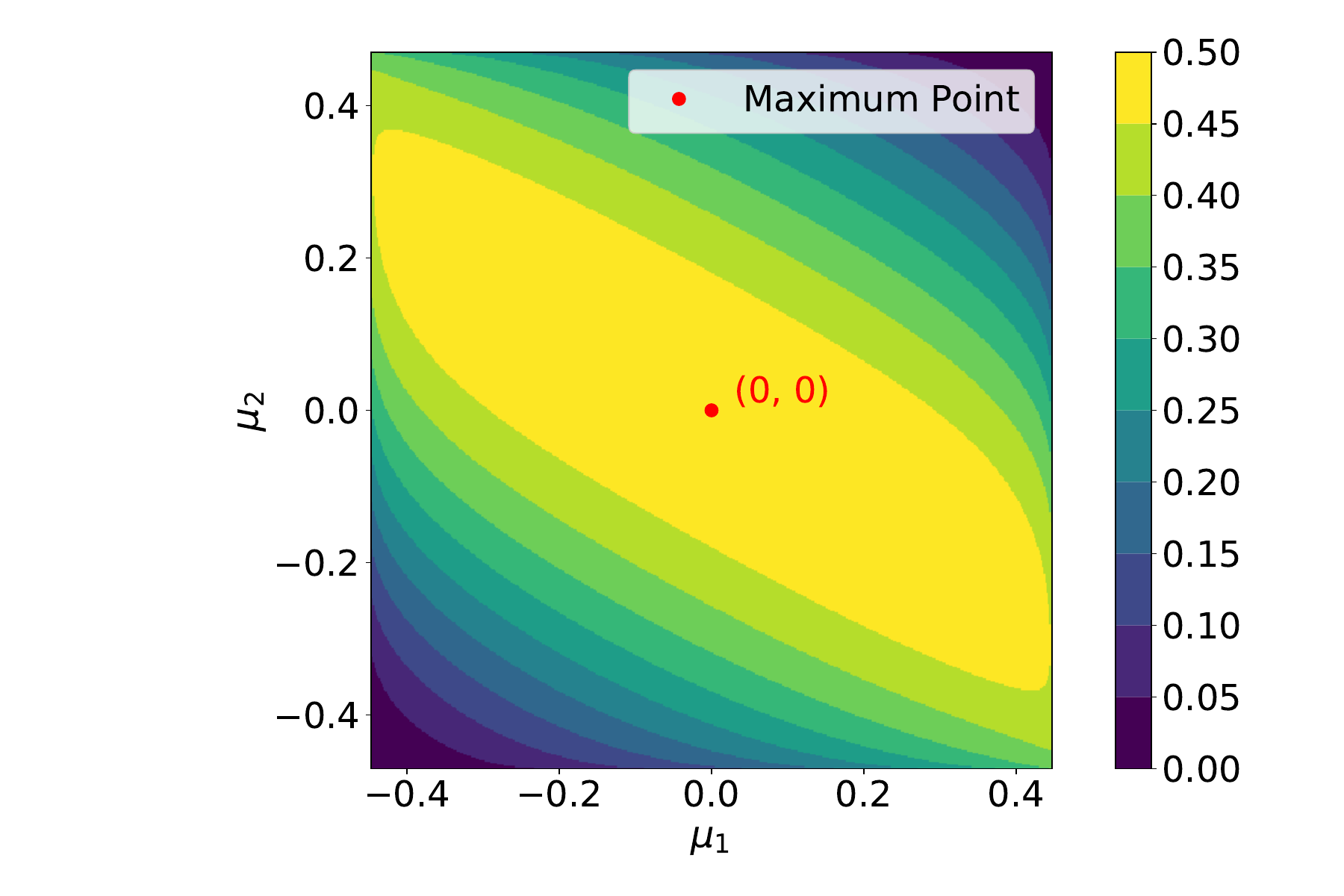}
		\caption{Heatmap of $\KL(\mathcal{N}_1 \| \mathcal{N}_2)$ with respect to $({\mu}_1, {\mu}_2)$, where $\N{N}_1$ satisfies ${\sigma}_1^{2} = w_2\left( 2 \Delta_1 - {\mu}_1^{2} \right)$ and $\N{N}_2$ satisfies $\dfrac{1}{{\sigma}_2^{2}} = \dfrac{ w_2\left[ 2 \Delta_2 - \log \left( {1 + {\mu}_2^{2}} \right) \right]  }{1 + {\mu}_2^{2}}$, with $\Delta_1 = \Delta_2 = 0.1$.}
		\label{fig_KL12_mu}
	\end{center}
\end{figure}

\textbf{Two-dimensional Case.}
We present a concrete example to illustrate the condition for achieving the supremum in Theorem~\ref{theorem_KLSupremumForGeneral} in the two-dimensional case. 

Specifically, we examine the relationship between $\N{N}_1$ and $\N{N}_2$ when $\KL\left( \N{N}_1 \, \| \, \N{N}_2 \right)$ reaches its supremum under the constraints
\begin{equation*}
	\left\{
	\begin{aligned}
		\KL\left(\N{N}(\vect{\mu}_1, \bm{\Sigma}_1) \, \| \, \mathcal{N}(\vect{0}, \bm{I}) \right) &= \Delta_1 \\
		\KL\left( \mathcal{N}(\vect{0}, \bm{I}) \,\|\, \N{N}(\vect{\mu}_2, \bm{\Sigma}_2) \right) &= \Delta_2
	\end{aligned}
	\right.
\end{equation*}
where $\N{N}_i = \N{N}(\vect{\mu}_i, \bm{\Sigma}_i)$ for $i \in \{1, 2\}$ denotes two two-dimensional Gaussian distributions and $\Delta_1 = \Delta_2 = 0.1$.

Figure~\ref{fig_TwoDGaussian} displays one possible realization of the heatmaps and surfaces of the probability density functions of $\N{N}_1$ and $\N{N}_2$ under the condition $\bm{Q} = \bm{I}$. 
By Theorem~\ref{theorem_KLSupremumForGeneral}, we have 
\begin{equation*}
	\bm{\Sigma}_1 = \left[ 
	\begin{aligned}
		w_2(0.2)~ &  \\
		& 1 
	\end{aligned} 
	\right]
\end{equation*}
and 
\begin{equation*}
	\bm{\Sigma}_2 = \left[ 
	\begin{aligned}
		w_2(0.2)^{-1}~ &  \\
		& 1 
	\end{aligned} 
	\right]
\end{equation*} Since $w_2(0.2) > 1$, in Figure~\ref{fig_TwoDGaussian} the distribution $\N{N}_1$ is stretched along the $x$ axis while $\N{N}_2$ is compressed along the $x$ axis compared with $\N{N}(\vect{0}, \bm{I})$, resulting in elliptical contour lines whose major axes are mutually orthogonal. 

\begin{figure*}[ht]
	\vskip 0.2in
	\begin{center}
		\centering
		\includegraphics[width=0.99 \linewidth]{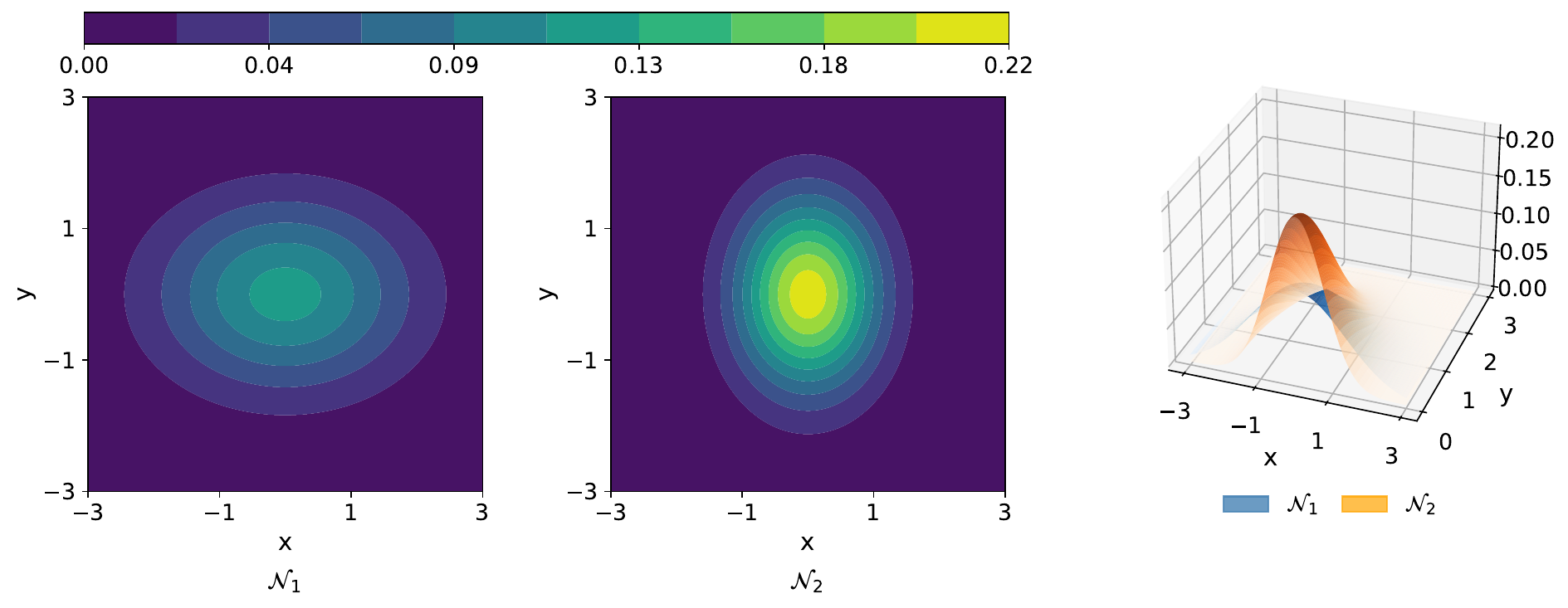}
		\caption{Heatmaps and surfaces of the two-dimensional Gaussian probability density functions of $\N{N}_1$ and $\N{N}_2$ when $\KL\left( \N{N}_1 \, \| \, \N{N}_2 \right)$ attains its supremum and $\bm{Q} = \bm{I}$. $\N{N}_1$ and $\N{N}_2$ satisfy the constraints $\KL\left(\N{N}(\vect{\mu}_1, \bm{\Sigma}_1) \, \| \, \mathcal{N}(\vect{0}, \bm{I}) \right) = \Delta_1 $ and $\KL\left( \mathcal{N}(\vect{0}, \bm{I}) \,\|\, \N{N}(\vect{\mu}_2, \bm{\Sigma}_2) \right) = \Delta_2$, respectively, with $\Delta_1 = \Delta_2 = 0.1$.}
		\label{fig_TwoDGaussian}
	\end{center}
\end{figure*}

\subsubsection{The Experiments on Lemma~\ref{lem_MaxH}}
Finally, we conduct numerical experiments to substantiate our core Lemma~\ref{lem_MaxH}. Details are shown in Appendix~\ref{appendix_HExperiment}.
  
\section{Discussions and Applications} \label{section_discussion}

\subsection{Comparison with Existing Results}
Theorem~\ref{theorem_KLSupremumForGeneral} establishes the supremum of $\KL(\mathcal{N}_1 \| \mathcal{N}_3)$ under fixed values of $\KL(\mathcal{N}_1 \| \mathcal{N}_2) = \Delta_1$ and $\KL(\mathcal{N}_2 \| \mathcal{N}_3) = \Delta_2$. In contrast, the original Theorem~4 of \cite{zhang2023PropertiesKullbackleiblerDivergence} provides a loose upper bound under the constraints $\KL(\mathcal{N}_1 \| \mathcal{N}_2) \leq \epsilon_1$ and $\KL(\mathcal{N}_2 \| \mathcal{N}_3) \leq \epsilon_2$. In principle, the constraints on $\KL(\mathcal{N}_1 \| \mathcal{N}_2)$ and $\KL(\mathcal{N}_2 \| \mathcal{N}_3)$ are the same in the problem.

Crucially, our bound is \emph{tight} and \emph{attainable}, the supremum is achieved by explicit Gaussian parameters (see Theorem~\ref{theorem_KLSupremumForGeneral})

For small divergence constraints, the asymptotic bound $\epsilon_1 + \epsilon_2 + 2\sqrt{\epsilon_1 \epsilon_2} + o(\epsilon_1) + o(\epsilon_2)$(Theorem~\ref{theorem_KLSupremumForGeneral}) is strictly tighter than the bound $3\epsilon_1 + 3\epsilon_2 + 2\sqrt{\epsilon_1 \epsilon_2} + o(\epsilon_1) + o(\epsilon_2)$ in the original Theorem~5 in \cite{zhang2023PropertiesKullbackleiblerDivergence}.
 
The main results of this work is achieved through three critical contributions.
\begin{enumerate}
	\item For Problem~\ref{problem_mu}, we provide a refined analysis that fully characterizes the supremum together with its attainment conditions.
	\item We show the conditions for achieve the optimum of Problem ~\ref{problem_mu} and ~\ref{problem_sigma} are the same. Therefore, the supremum of the whole relaxed triangle inequality can be attained.
	\item For the objective function $H(x, y; \Delta_1, \Delta_2)$, we develop a concise and tractable approach to its global optimization by leveraging the first order necessary conditions for an interior maximum.
\end{enumerate}

Moreover, we acknowledge that the proof of Theorem~\ref{theorem_KLSupremumForGeneral} closely follows the overall structure of the proof of the original  Theorem~4 in \cite{zhang2023PropertiesKullbackleiblerDivergence}. The main novelty lies in the proof of Lemma~\ref{lem_KLSupremumForNormal}, which constitutes the most technical part of our work. Nevertheless, it builds directly upon several key results from \cite{zhang2023PropertiesKullbackleiblerDivergence}, which enable us to substantially streamline the overall argument.
\subsection{Applications}
In \cite{zhang2023PropertiesKullbackleiblerDivergence}, Zhang \textit{et al.} have discussed several applications of relaxed triangle inequality of KL divergence between Gaussian distributions. The theorem presented in this paper improves the bound in  \cite{zhang2023PropertiesKullbackleiblerDivergence}, and hence further strengthen the theoretical foundation in applications. 
In the following, we discuss two applications in out-of-distribution detection and reinforcement learning, respectively.
\subsubsection{Application in Out-of-distribution Detection with Flow-based Models}
\noindent 
Flow-based generative models define a bijective and differentiable mapping $z = f(x)$ from the data space $\mathcal{X}$ to a latent space $\mathcal{Z}$. 
The training objective of flow-based model is usually maximizing likelihood estimation. 
However, a well-documented yet counterintuitive phenomenon arises after training. The model often assigns higher log-likelihood values to out-of-distribution data (OOD data, also referred to as anomalies) \cite{nalisnick2018deep, shafaei2018digitnotcat, choi2018generative, vskvara2018generative,nalisnick2019detecting,whyflowfailood}, yet fails to generate realistic OOD samples from the corresponding regions in the latent space.
For example, Glow \cite{kingma2018glow} trained on CIFAR-10 dataset assigns higher likelihoods for SVHN dataset, but cannot generate new SVHN-like data.

Building on the approximate symmetry and relaxed triangle inequality proved in \cite{zhang2023PropertiesKullbackleiblerDivergence}, Zhang \textit{et al.} provided a rigorous explanation for this long-standing counterintuitive behavior~\cite{zhang2024KullbackLeiblerDivergenceBasedOutofDistribution}. The theoretical analysis is summarized as follows.

Let $p_Z^r$ be the most commonly used Gaussian prior distribution for flow-based model $f$.
Let $p_X$, $q_X$ be the distributions of ID and OOD data, respectively, $p_Z$, $q_Z$ be the distributions of representations of ID and OOD data, respectively. $p_X^r$ be the model induced distribution such that $Z_r\sim p_Z^r$ and $X_r=f^{-1}(Z_r)\sim p^r_X$.
For flow-based models, maximizing likelihood estimation is equal to minimizing  forward KL divergence $\KL(p_X||p^r_X)$ \cite{papamakarios2019flow_model_survey, goodfellow2016deep}.

From empirical observations and normality tests, both $p_Z$ and $q_Z$ are approximately Gaussian distributions for many real-world OOD problems. 
We have the following two conditions.
\begin{enumerate}
	\item Since flow-based model preserves KL divergence, $\mathrm{KL}(p_X \| p_X^r)=\mathrm{KL}(p_Z \| p_Z^r)$ is small due to maximum likelihood training.
	\item $\mathrm{KL}(p_Z \| q_Z) = \mathrm{KL}(p_X \| q_X)$ is supposed to be arbitrarily large, since ID and OOD data originate from distinct distributions.
\end{enumerate}
Now, suppose on the contrary that $\KL(p_Z^r \| q_Z)$ were also small. Then, by applying relaxed triangle inequality with the substitution $(\N{N}_1, \N{N}_2, \N{N}_3) = (p_Z, p_Z^r, q_Z)$, we would conclude that $\mathrm{KL}(p_Z \| q_Z)$ must be small, contradicting the fact that $\mathrm{KL}(p_Z \| q_Z) = \mathrm{KL}(p_X \| q_X)$ is large. 
Furthermore, according to the approximate symmetry theorem presented in \cite{zhang2023PropertiesKullbackleiblerDivergence}, we can know both $\KL(p_Z^r \| q_Z)$ and $\KL(q_Z || p_Z^r)$ are large. This explains why OOD samples cannot be generated by sampling from the prior, even when the likelihoods of ID and OOD data overlap. More analysis can be referred to \cite{zhang2024KullbackLeiblerDivergenceBasedOutofDistribution}.

Since Theorem~\ref{theorem_KLSupremumForGeneral} is the most optimized improvement of the relaxed triangle inequality \cite{zhang2023PropertiesKullbackleiblerDivergence}, the above analysis can also be regarded as an application instance of Theorem~\ref{theorem_KLSupremumForGeneral}.

\subsubsection{Application in Safe Reinforcement Learning}
\noindent
Liu \textit{et al.}~\cite{liu2022ConstrainedVariationalPolicy} propose an Expectation-Maximization-style approach for learning safe policies in reinforcement learning. They employ the original relaxed triangle inequality in \cite{zhang2023PropertiesKullbackleiblerDivergence} to extend single-step safety guarantee to multiple steps. In \cite{liu2022ConstrainedVariationalPolicy}, the authors suppose  $\epsilon_1=\epsilon_2$ and the bound in the inequality becomes $3\epsilon_1 + 3\epsilon_2 + 2 \sqrt{\epsilon_1 \epsilon_2} + o(\epsilon_1) + o(\epsilon_2)=8\epsilon_1+o(\epsilon_1)$.

In this paper, Theorem~\ref{theorem_KLUpperBoundForEpsilon} tightens the bound in the relaxed triangle inequality to the optimal upper bound $\epsilon_1 + \epsilon_2 + 2 \sqrt{\epsilon_1 \epsilon_2} + o(\epsilon_1) + o(\epsilon_2) = 4 \varepsilon_1 + o(\varepsilon_1)$ for $\epsilon_1=\epsilon_2$, achieving a $50\%$ reduction compared to the bound $8 \epsilon_1 + o(\epsilon_1)$. This improvement substantially strengthens the theoretical guarantee for multiple-step safety guarantee in reinforcement learning.

\section{Conclusion} \label{section_conclusion}
\noindent
In this paper, we investigate the relaxed triangle inequality for the Kullback--Leibler (KL) divergence between Gaussian distributions.  
First, we rigorously prove that, given fixed values of $\KL(\N{N}_1 \| \N{N}_2) = \Delta_1$ and $\KL(\N{N}_2 \| \N{N}_3) = \Delta_2$, the dimension-free supremum equals the closed-form expression $\frac{1}{2}\left[ w_2(2 \Delta_1) - 1 \right]\left[ w_2(2 \Delta_2) - 1 \right] + \Delta_1 + \Delta_2$. We explicitly characterize the necessary and sufficient conditions under which this supremum is attained.  
Then, to improve practical application, we consider small divergence values  $\KL(\N{N}_1 \| \N{N}_2) = \epsilon_1$ and $\KL(\N{N}_2 \| \N{N}_3) = \epsilon_2$, and show that the supremum of $\KL(\N{N}_1 \| \N{N}_3)$ is $\epsilon_1 + \epsilon_2 + 2 \sqrt{\epsilon_1 \epsilon_2} + o(\epsilon_1) + o(\epsilon_2)$.  
Finally, the theoretical results developed in this paper extend the applicability of the KL divergence in various domains, including out-of-distribution detection with flow-based models and safe reinforcement learning.

\bibliography{refs} 
\bibliographystyle{IEEEtran}

\begin{IEEEbiography}[{\includegraphics[width=1in,height=1.25in,clip,keepaspectratio]{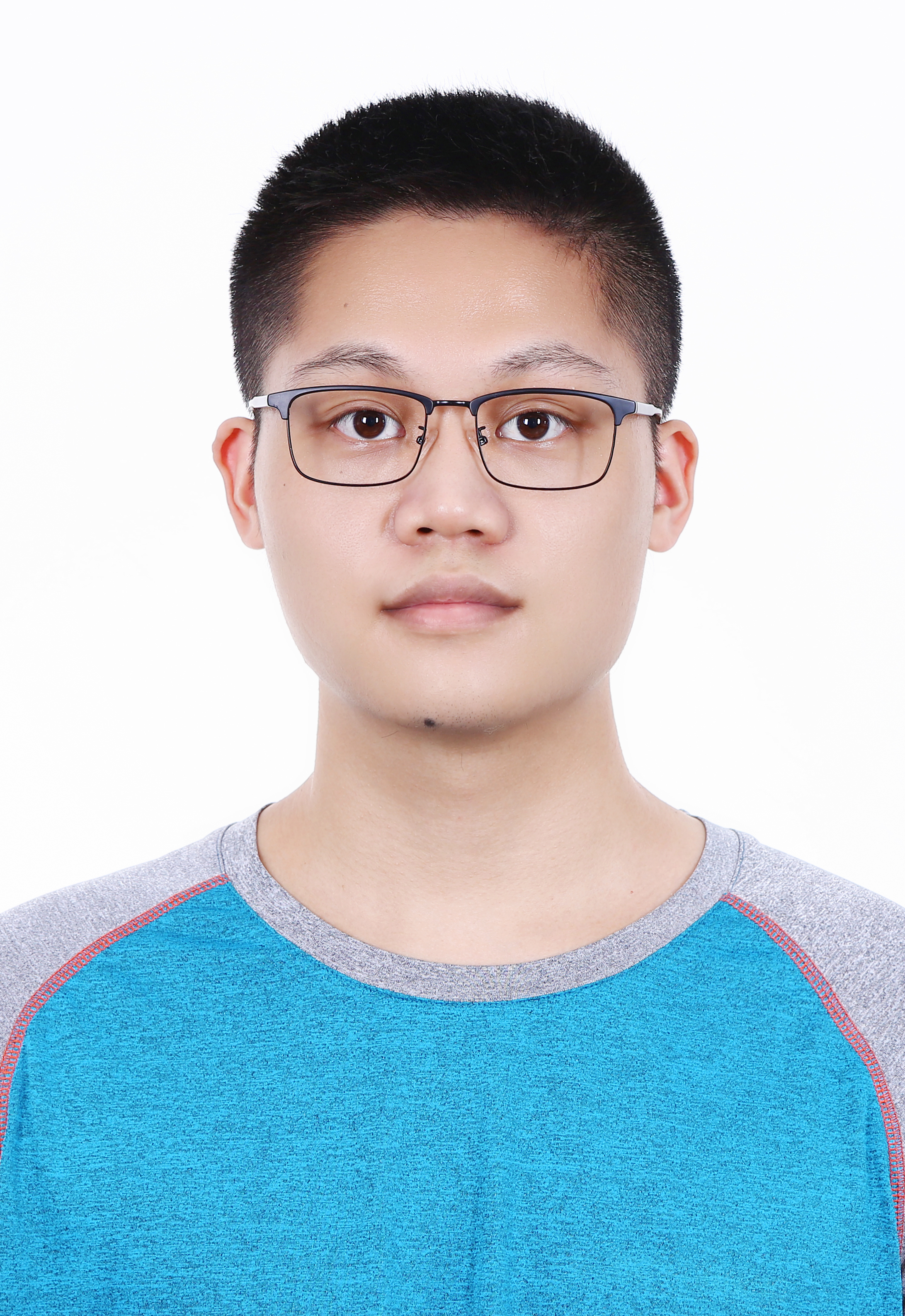}}]{Shiji Xiao}
    received the BEng and MEng degrees in aerospace engineering from Beihang University, Beijing, China. He is currently a first year PhD student in computer science at Hunan University, Changsha, Hunan, China. His research interests include artificial intelligence.
\end{IEEEbiography}

\begin{IEEEbiography}[{\includegraphics[width=1in,height=1.25in,clip,keepaspectratio]{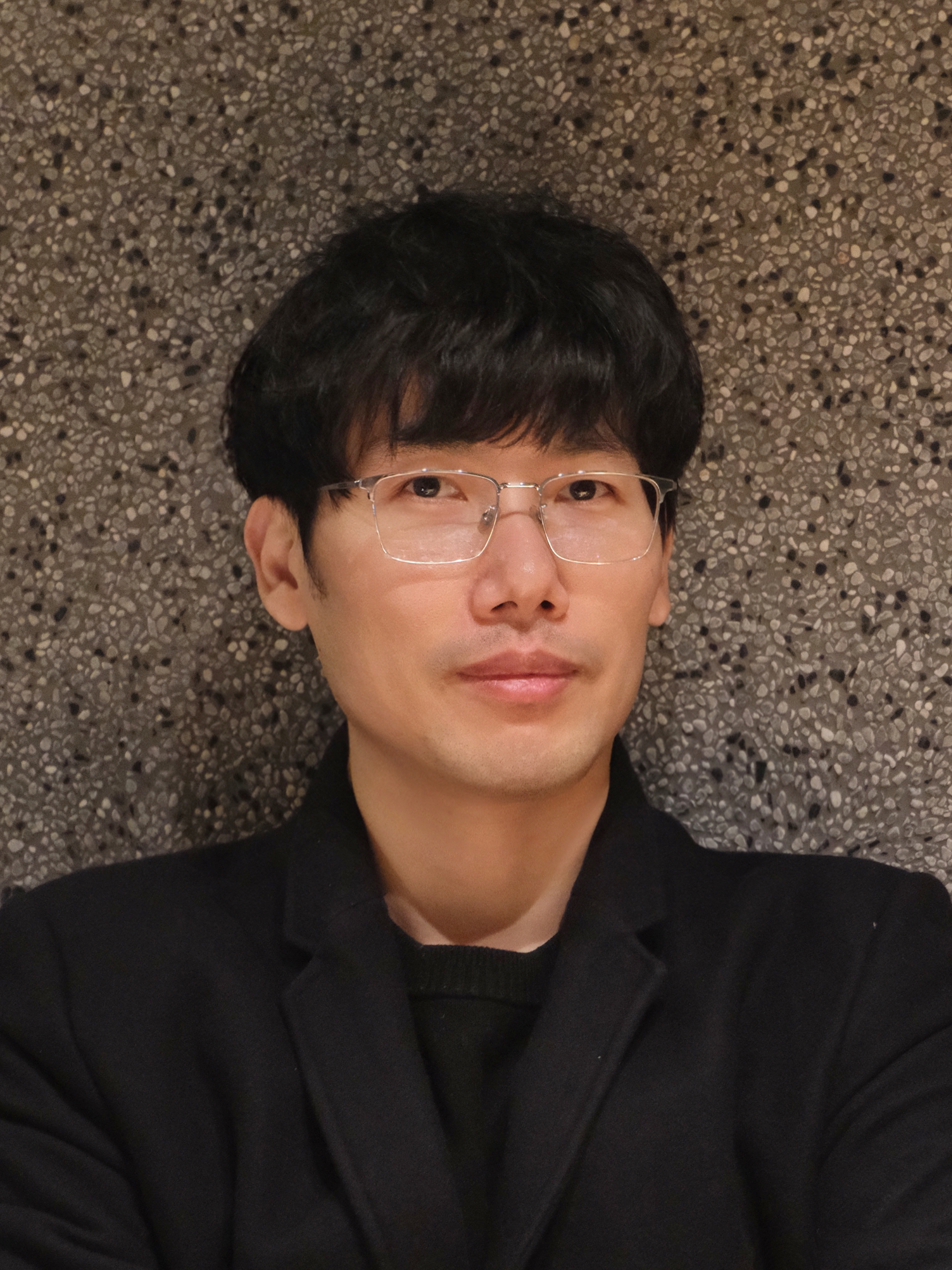}}]{Yufeng Zhang}
     received the PhD degree from the College of Computer, National University of Defense Technology, Changsha, China, in 2013. He is a full professor with the College of Computer Science and Electronic Engineering, Hunan University, Changsha, Hunan, China. His research interests include artificial intelligence and software engineering.
\end{IEEEbiography}

\begin{IEEEbiography}[{\includegraphics[width=1in,height=1.25in,clip,keepaspectratio]{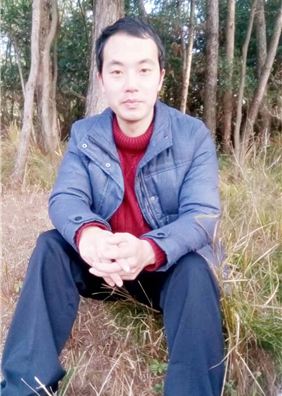}}]{Chubo Liu}
    (Member, IEEE) received the B.S. and Ph.D. degrees in computer science and technology from Hunan University, China, in 2011 and 2016, respectively. Currently, he is a Full Professor of computer science and technology with Hunan University. He has published over 40 papers in journals and conferences such as IEEE Transactions on Parallel and Distributed Systems, IEEE Transactions on Cloud Computing, IEEE Transactions on Mobile Computing, IEEE Transactions on Industrial Informatics, IEEE Internet of Things Journal, ACM Transactions on Modeling and Performance Evaluation of Computing Systems, Theoretical Computer Science, ISCA, DAC, and NPC. He won the IEEE TCSC Early Career Researcher (ECR) Award in 2019. He is a member of ACM.
\end{IEEEbiography}

\begin{IEEEbiography}[{\includegraphics[width=1in,height=1.25in,clip,keepaspectratio]{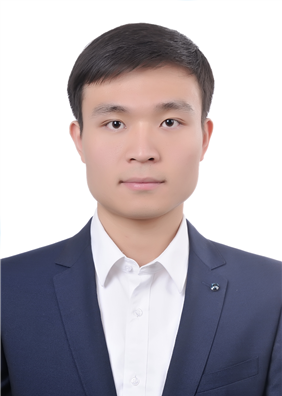}}]{Yan Ding}
    received the PhD degree in computer science from Hunan University, China, in 2022. He is currently an associate professor with Hunan University. He has published dozens of papers in journals and conferences. His research interests include parallel computing, mobile edge computing, Big Data, artificial intelligence, and architecture. He received the Outstanding Paper Award in IEEE ISPA 2019.
\end{IEEEbiography}

\begin{IEEEbiography}[{\includegraphics[width=1in,height=1.25in,clip,keepaspectratio]{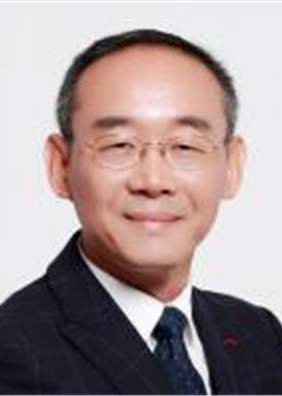}}]{Keqin Li}
    is a SUNY Distinguished Professor of computer science. His research interests are mainly in the areas of design and analysis of algorithms, parallel and distributed computing, and computer networking. He has contributed extensively to approximation algorithms, parallel algorithms, job scheduling, task dispatching, load balancing, performance evaluation, dynamic tree embedding, scalability analysis, parallel computing using optical interconnects, wireless networks, and optical networks. His current research interests include power-aware computing, location management in wireless communication networks, lifetime maximization in sensor networks, and file sharing in peer-to-peer systems.
Dr. Li has published over 225 journal articles, book chapters, and research papers in refereed international conference proceedings.
\end{IEEEbiography}

\begin{IEEEbiography}[{\includegraphics[width=1in,height=1.25in,clip,keepaspectratio]{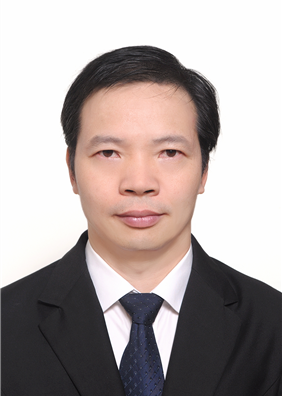}}]{Kenli Li}
    received the PhD degree in computer science from Huazhong University of Science and Technology, China, in 2003. He was a visiting scholar at the University of Illinois at Urbana- Champaign from 2004 to 2005. He is currently a full professor of computer science and technology at Hunan University and the deputy director in the National Supercomputing Center in Changsha. His major research areas include parallel computing, high-performance computing, and grid and cloud computing. He has published more than 130 research papers in international conferences and journals such as IEEE Transactions on Computers, IEEE Transactions on Parallel and Distributed Systems, Journal of Parallel and Distributed Computing, ICPP, CCGrid. He is an outstanding member of CCF. He serves on the editorial board of the IEEE Transactions on Computers. He is a senior member of the IEEE.
\end{IEEEbiography}

\appendices  
\section{Lemma~\ref{lem_MaxMuFun} and Its Proof} \label{section_MaxMuFun}

\begin{lemma} \label{lem_MaxMuFun}
	Let $\bm{\Sigma} \in \mathbb{S}^n_{++}$ be a variable positive definite matrix and $\vect{\mu_1}, \vect{\mu_2} \in \mathbb{R}^n$ be variable vectors satisfying
	\begin{equation} \label{eq_sigmaConstraints}
		-\log|\bm{\Sigma}| + \operatorname{Tr}(\bm{\Sigma})  = n + y
	\end{equation}
	and
	\begin{equation} \label{eq_muConstraints}
		\left\{
		\begin{array}{c}
			\vect{\mu_1}^\top \vect{\mu_1} = \varepsilon_1 \\
			\vect{\mu_2}^\top \bm{\Sigma} \vect{\mu_2} = \varepsilon_2
		\end{array}
		\right.
	\end{equation}
	where $\varepsilon_1 \geq 0$, $\varepsilon_2 \geq 0$, and $y \geq 0$ are fixed constants,  then
	\begin{equation} \label{ieq_KernelMuSigmaIeq}
		(\vect{\mu_2} - \vect{\mu_1})^\top \bm{\Sigma} (\vect{\mu_2} - \vect{\mu_1}) 
		\leq \left( \sqrt{w_2(y) \varepsilon_1} + \sqrt{\varepsilon_2} \right)^2
	\end{equation}
	Equality holds when
	\begin{equation} \label{eq_KernelMuSigmaConditions}
		\left\{ 
		\begin{aligned}
			\vect{\mu}_1 &= x_{1} \vect{e}_{1} \\
			\vect{\mu}_2 &= y_{1} \vect{e}_{1} \\
			\bm{\Sigma} &= \bm{Q} \operatorname{diag} \left(w_2(y), 1, \dots, 1 \right) \bm{Q}^{\top}
		\end{aligned} 
		\right.
	\end{equation}
	where $\bm{Q} \operatorname{diag} \left(w_2(y), 1, \dots, 1 \right) \bm{Q}^{\top}$ is the eigenvalue decomposition of $\bm{\Sigma}$, $(x_1, y_1) = \pm \left(  \sqrt{\varepsilon_1}, - \sqrt{\frac{\varepsilon_2}{w_2(y)}} \right)$,  and $\vect{e}_{1}$ is the first column of $\bm{Q}$.
\end{lemma}

To prove Lemma~\ref{lem_MaxMuFun}, we first establish Lemma~\ref{lem_KernelXYInequality} in Subsection~\ref{subsection_KernelXYInequality} and then prove Lemma~\ref{lem_MaxMuFun} in Subsection~\ref{subsection_MaxMuFun}.

\subsection{Lemma~\ref{lem_KernelXYInequality} and Its Proof} \label{subsection_KernelXYInequality}

\begin{lemma} \label{lem_KernelXYInequality}
	Suppose the variables $x_i$ and $y_i$ for $i = 1, \dots, n$ satisfy
	\begin{equation} \label{eq_xyConstraints}
		\left\{
		\begin{array}{c}
			\sum\limits_{i=1}^n x_i^2 = \varepsilon_1 \\
			\sum\limits_{i=1}^n \lambda_i y_i^2 = \varepsilon_2
		\end{array}
		\right.
	\end{equation}
	where $\varepsilon_1, \varepsilon_2 \geq 0$ and $\lambda_i > 0$ for $i = 1, \dots, n$ are fixed constants, then
	\begin{equation} \label{ieq_lambdaxyIeq}
		\sum_{i=1}^n \lambda_i (y_i - x_i)^2 \leq \left( \sqrt{\lambda_{\max} \varepsilon_1} + \sqrt{\varepsilon_2} \right)^2 
	\end{equation}
	If $\varepsilon_1 = 0$, equality holds for all admissible $x_i$ and $y_i$ satisfying \eqref{eq_xyConstraints}.
	If $\varepsilon_2 = 0$, equality holds if and only if
	\begin{equation*}
		x_i = 0 \text{ for all } i \in U \setminus I(\lambda_{\max})
	\end{equation*}
	If $\varepsilon_1 > 0$ and $\varepsilon_2 > 0$, equality holds if and only if
	\begin{equation*}
		\begin{cases}
			y_i = k x_i & \text{for } i \in I(\lambda_{\max}) \\
			y_i = x_i = 0 & \text{for } i \in U \setminus I(\lambda_{\max})
		\end{cases}
	\end{equation*}
	where $U = \left\lbrace 1, \dots, n \right\rbrace $, $\lambda_{\max} = \max_{i=1}^n \lambda_i$, $I(\lambda_{\max}) = \left\{i \mid \lambda_i = \lambda_{\max}, \, i \in U \right\}$ and $k = -\sqrt{\dfrac{\varepsilon_2}{\lambda_{\max} \varepsilon_1}}$.
\end{lemma}

\begin{IEEEproof} \label{app_proof_lem_P>0_Cauchy}
	By the conditions in \eqref{eq_xyConstraints} and the definition of $\lambda_{\max}$ and $I(\lambda_{\max})$, we have $(\lambda_{\max} - \lambda_{i}) > 0$ for all $i \in U \setminus I(\lambda_{\max})$ and thus
	\begin{equation*}
		\begin{aligned}
			&\sum_{i=1}^n {\lambda_{i} {x_i}^2} - \lambda_{\max} \varepsilon_1 \\
			=& \sum_{i=1}^n {\lambda_{i} {x_i}^2} - \lambda_{\max} \sum_{i=1}^n {{x_i}^2} = -\sum_{i=1}^n {(\lambda_{\max} - \lambda_{i}) {x_i}^2} \\
			=& - \sum_{i \in U \setminus I(\lambda_{\max})}{(\lambda_{\max} - \lambda_{i}) {x_i}^2} \leq 0
		\end{aligned}
	\end{equation*}
	that is,
	\begin{equation}
		\sum_{i=1}^n {\lambda_{i} {x_i}^2} \leq \lambda_{\max} \varepsilon_1 \label{ieq:lambda_x}
	\end{equation}
	with equality if and only if
	\begin{equation}
		x_i = 0, ~ i \in U \setminus I(\lambda_{\max}) \label{eq:lambda_x}
	\end{equation}
	Similarly, we have 
	\begin{equation*}
		\begin{aligned}
			& \sum_{i=1}^n {\lambda_{i}^2 {y_i}^2} - \lambda_{\max} \varepsilon_2 \\
			=& \sum_{i=1}^n {\lambda_{i}^2 {y_i}^2} - \lambda_{\max} \sum_{i=1}^n { \lambda_{i} {y_i}^2} = -\sum_{i=1}^n {(\lambda_{\max} - \lambda_{i}) \lambda_{i} {y_i}^2} \\
			=& - \sum_{i \in U \setminus I(\lambda_{\max})}{(\lambda_{\max} - \lambda_{i}) \lambda_{i} {y_i}^2} \leq 0
		\end{aligned}
	\end{equation*}
	that is,
	\begin{equation}
		\sum_{i=1}^n {\lambda_{i}^2 {y_i}^2} \leq \lambda_{\max} \varepsilon_2 \label{ieq:lambda_y}
	\end{equation}
	with equality if and only if
	\begin{equation}
		y_i = 0, ~ i \in U \setminus I(\lambda_{\max}) \label{eq:lambda_y}
	\end{equation}
	
	If $\varepsilon_1 = 0$, condition $\sum_{i=1}^n x_i^2 = \varepsilon_1 = 0$ together with $\lambda_i > 0$ for $i \in U$ implies $x_i = 0$ for all $i \in U$. Moreover, applying the condition $\sum_{i=1}^n \lambda_i y_i^2 = \varepsilon_2$, we have
	\begin{equation*}
		\sum_{i=1}^n \lambda_i (y_i - x_i)^2 = \sum_{i=1}^n \lambda_i y_i^2 = \varepsilon_2 = \left( \sqrt{\lambda_{\max} \varepsilon_1} + \sqrt{\varepsilon_2} \right)^2
	\end{equation*}
	which shows that the inequality \eqref{ieq_lambdaxyIeq} holds with equality for all admissible $x_i$ and $y_i$ satisfying \eqref{eq_xyConstraints}.
	
	If $\varepsilon_2 = 0$, we similarly obtain $y_i = 0$ for all $i \in U$.
	Thus, combining the condition $\sum_{i=1}^n x_i^2 = \varepsilon_1$, we obtain
	\begin{equation*}
		\begin{aligned}
			\sum_{i=1}^n \lambda_i (y_i - x_i)^2 =& \sum_{i=1}^n \lambda_i x_i^2 
			\\
			\leq& \lambda_{\max} \varepsilon_1 \text{ ~ (by \eqref{ieq:lambda_x})} \\
			=& \left( \sqrt{\lambda_{\max} \varepsilon_1} + \sqrt{\varepsilon_2} \right)^2
		\end{aligned}
	\end{equation*}
    The equality  holds if and only if
	\begin{equation*}
		x_i = 0 \text{ for all } i \in U \setminus I(\lambda_{\max}) .
	\end{equation*} 
	
	If $\varepsilon_1 > 0$ and $\varepsilon_2 > 0$, we have
	\begin{equation} \label{ieq:lambda_xy}
		\begin{aligned}
			&\sum_{i=1}^n {(-x_i) \lambda_i  y_i} \\
			\leq& \sqrt{\sum_{i=1}^n {(-x_i) ^ 2}} \sqrt{\sum_{i=1}^n {(\lambda_i y_i) ^ 2}} \text{ (Cauchy--Schwarz inequality)} \\
			\leq& \sqrt{\varepsilon_1} \sqrt{\lambda_{\max} \varepsilon_2 } \text{ (by \eqref{eq_xyConstraints} and \eqref{ieq:lambda_y})}
		\end{aligned} 
	\end{equation}
	and then
	\begin{equation*} 
		\begin{aligned}
			&\sum_{i=1}^n \lambda_i (y_i - x_i)^2 \\
			=& \sum\limits_{i=1}^n \lambda_i y_i^2 + \sum_{i=1}^n {\lambda_{i} {x_i}^2} + \sum_{i=1}^n {2 \lambda_i (-x_i) y_i} \\
			\leq& \varepsilon_2 + \lambda_{\max} \varepsilon_1 + 2\sqrt{\varepsilon_1} \sqrt{\lambda_{\max} \varepsilon_2 } \text{ (by \eqref{eq_xyConstraints}, \eqref{ieq:lambda_x}, and \eqref{ieq:lambda_xy})} \\
			=& \left( \sqrt{\lambda_{\max} \varepsilon_1} + \sqrt{\varepsilon_2} \right)^2 \\
		\end{aligned}
	\end{equation*}
	that is,
	\begin{equation*}
		\sum_{i=1}^n \lambda_i (y_i - x_i)^2 \leq \left( \sqrt{\lambda_{\max} \varepsilon_1} + \sqrt{\varepsilon_2} \right)^2
	\end{equation*}
	Equality holds if and only if \eqref{eq_xyConstraints}, \eqref{eq:lambda_x}, \eqref{eq:lambda_y} and the following condition hold
	\begin{equation}
		\label{eq:lambda_xy}
		\exists \bar{k} \geq 0 \text{ such that } - x_i = \bar{k} \lambda_i y_i ,~  i \in U 
	\end{equation}
	Then we have
	\begin{align*}
		\varepsilon_1 &= \sum_{i \in U} x_i^2  \text{ (by \eqref{eq_xyConstraints})}\\
		&= \sum_{i \in I(\lambda_{\max})} x_i^2  \text{ (by \eqref{eq:lambda_x})} \\
		&= \sum_{i \in I(\lambda_{\max})} (-\bar{k} \lambda_i y_i)^2 \text{ (by \eqref{eq:lambda_xy})} \\
		&= \bar{k}^2 \lambda_{\max} \sum_{i \in I(\lambda_{\max})} \lambda_i y_i^2 \\
		&= \bar{k}^2 \lambda_{\max} \sum_{i \in U} \lambda_i y_i^2 \text{ (by \eqref{eq:lambda_y})} \\
		&= \bar{k}^2 \lambda_{\max} \varepsilon_2  \text{ (by \eqref{eq_xyConstraints})}
	\end{align*}
	Since $\bar{k} \geq 0$, it follows that $\bar{k} = \sqrt{\frac{\varepsilon_1}{\lambda_{\max} \varepsilon_2}}$. Thus, from \eqref{eq_xyConstraints}, \eqref{eq:lambda_x}, \eqref{eq:lambda_y}, and \eqref{eq:lambda_xy}, we obtain
	\begin{equation}
		\left\{
		\begin{aligned}
			&y_i = k x_i, && i \in I(\lambda_{\max}) \\
			&y_i = x_i = 0, && i \in U \setminus I(\lambda_{\max}) 
		\end{aligned}
		\right.
		\label{eq_FinalForm}
	\end{equation}
	where $k = - \left( \bar{k} \lambda_{\max} \right)^{-1}  = -\sqrt{\frac{\varepsilon_2}{\lambda_{\max} \varepsilon_1}}$. Moreover, it is straightforward to verify that \eqref{eq_FinalForm} implies \eqref{eq_xyConstraints}, \eqref{eq:lambda_x}, \eqref{eq:lambda_y}, and \eqref{eq:lambda_xy}. Hence, \eqref{eq_FinalForm} is equivalent to these four conditions. Therefore, when $\varepsilon_1 > 0$ and $\varepsilon_2 > 0$, inequality \eqref{ieq_lambdaxyIeq} holds and equality holds if and only if \eqref{eq_FinalForm} is satisfied with $k = -\sqrt{\dfrac{\varepsilon_2}{\lambda_{\max} \varepsilon_1}}$.
	
	The lemma is thus proved for all cases.
\end{IEEEproof}

\subsection{Proof of Lemma~\ref{lem_MaxMuFun}} \label{subsection_MaxMuFun}

\begin{IEEEproof}
	Let $\lambda_{i} > 0$ $(i=1,\dots,n)$ be the $n$ eigenvalues of $\bm{\Sigma}$ with corresponding unit eigenvectors $\vect{e}_1,\dots,\vect{e}_n$, and denote $\lambda_{\max} = \max_{i=1}^n \lambda_i$.
	
	From Equations~(C-25) and~(C-26) in~\cite[Appendix C]{zhang2023PropertiesKullbackleiblerDivergence} and the condition \eqref{eq_sigmaConstraints}, we have
	\begin{equation*}
		-\log|\bm{\Sigma}| + \operatorname{Tr}(\bm{\Sigma}) = \sum_{i=1}^n \bigl( \lambda_i - \log \lambda_i \bigr) = n + y
	\end{equation*}
	Since the function $f(\lambda) = \lambda - \log \lambda$ satisfies $f(\lambda) \geq 1$ for all $\lambda > 0$, with equality if and only if $\lambda = 1$, we obtain
	\begin{align*}
		\lambda_{\max} - \log \lambda_{\max}
		&= n + y - \sum_{i \neq i_{\max}} \bigl( \lambda_i - \log \lambda_i \bigr) \\ 
		&\leq n + y - (n - 1) = y + 1
	\end{align*}
    where $i_{\max} \in U$ satisfying $\lambda_{i_{\max}} = \lambda_{\max}$.
	Then we obtain 
	\begin{equation} \label{eq_lambdaLessThanW2(y)}
		\lambda_{\max} \leq w_2(y)
	\end{equation}
	
	Meanwhile, $\{\vect{e}_{1},\dots,\vect{e}_{n}\}$ forms an orthonormal basis of $\mathbb{R}^n$ and 
	\begin{equation*}
		\vect{e}_{i}\cdot\vect{e}_{j}=\begin{cases}
			1 & i = j \\
			0 & i \neq j
		\end{cases}
	\end{equation*}
	We express $\vect{\mu_1}$ and $\vect{\mu_2}$ in this basis as $\vect{\mu_1} = \sum_{i=1}^n x_i \vect{e_i}$ and $\vect{\mu_2} = \sum_{i=1}^n y_i \vect{e_i}$.
	Applying the constraints \eqref{eq_muConstraints}, we obtain
	\begin{equation*}
		\vect{\mu_1}^\top\vect{\mu_1} = \sum_{i=1}^n x_i^2 = \varepsilon_1   
	\end{equation*}
	and
	\begin{equation*}
		\vect{\mu_2}^\top\bm{\Sigma}\vect{\mu_2} = \left(\sum\limits_{i=1}^n y_i \vect{e_i}\right)^\top\left(\sum\limits_{j=1}^n y_j \lambda_j \vect{e_j}\right) = \sum_{i=1}^n \lambda_i y_i^2 = \varepsilon_2  
	\end{equation*}
	That is,
	\begin{equation} \label{eq_xyEq}
		\left\{ \begin{array}{c}
			\sum\limits_{i=1}^n x_i^2 = \varepsilon_1\\
			\sum\limits_{i=1}^n \lambda_i y_i^2 = \varepsilon_2\\
		\end{array} \right.
	\end{equation}
	Consequently,
	\begin{equation*}
		\begin{aligned}
			&(\vect{\mu_2}-\vect{\mu_1})^\top \bm{\Sigma} (\vect{\mu_2}-\vect{\mu_1}) \\
			=& \sum_{i=1}^n \lambda_i (y_i-x_i)^2 \\
			\leq& \left(\sqrt{\lambda_{\max} \varepsilon_1}+\sqrt{\varepsilon_2}\right)^2 \text{ (by \eqref{eq_xyEq} and Lemma~\ref{lem_KernelXYInequality})} \\
			\leq & \left(\sqrt{w_2(y) \varepsilon_1}+\sqrt{\varepsilon_2}\right)^2 \text{ (by \eqref{eq_lambdaLessThanW2(y)})}
		\end{aligned}
	\end{equation*}
	which establishes inequality \eqref{ieq_KernelMuSigmaIeq}.
	
	When condition \eqref{eq_KernelMuSigmaConditions} holds, we have 
    \begin{equation*} 
		\lambda_i = \left\{ \begin{array}{cc}
			w_2(y) & \text{if } i = 1 \\
			1 & \text{if } i \neq 1
		\end{array} \right.
	\end{equation*}
    and the first column of $\bm{Q}$ is an eigenvector $\vect{e}_{1}$ corresponding to $\lambda_{1}$. Then
    \begin{equation*}
        \begin{aligned}
		      &-\log|\bm{\Sigma}| + \operatorname{Tr}(\bm{\Sigma}) \\
              =& \sum_{i=1}^n \left( \lambda_i - \log \lambda_i \right) =(y + 1) +  (n - 1) = n + y
		\end{aligned}
	\end{equation*}
	\begin{equation*} 
		\left\{
		\begin{aligned}
			\vect{\mu_1}^\top \vect{\mu_1} &= x_1^{2} = \varepsilon_1 \\
			\vect{\mu_2}^\top \bm{\Sigma} \vect{\mu_2} &= y_1^{2} \lambda_{1} = \varepsilon_2 \\
		\end{aligned}
		\right.
	\end{equation*}
	and
	\begin{equation*} 
		\begin{aligned}
			&(\vect{\mu_2} - \vect{\mu_1})^\top \bm{\Sigma} (\vect{\mu_2} - \vect{\mu_1}) = (y_1 - x_1) \vect{e_1}^\top \bm{\Sigma} (y_1 - x_1) \vect{e_1}  \\
			=& (y_1 - x_1)^{2} \lambda_{1} = \left( \sqrt{\varepsilon_2} + \sqrt{\lambda_{1} \varepsilon_1} \right)^2 = \left( \sqrt{w_2(y) \varepsilon_1} + \sqrt{\varepsilon_2} \right)^2
		\end{aligned}
	\end{equation*}
	That is when condition \eqref{eq_KernelMuSigmaConditions} holds, both constraints \eqref{eq_sigmaConstraints} and \eqref{eq_muConstraints} are satisfied and the equality in inequality \eqref{ieq_KernelMuSigmaIeq} can be attained
	
	The lemma is thus proved.
\end{IEEEproof}

\section{Lemma~\ref{lem_MaxSigmaFun} and Its New Proof} \label{section_MaxSigmaFun}

\begin{lemma} \cite[Appendix H]{zhang2023PropertiesKullbackleiblerDivergence} \label{lem_MaxSigmaFun}
	Let $\bm{\Sigma}_1, \bm{\Sigma}_2 \in \mathbb{S}^n_{++}$ be variable positive definite matrices satisfying
	\begin{equation*}
		\left\{\begin{aligned}
			-\log{|\bm{\Sigma}_1|} + \operatorname{Tr}(\bm{\Sigma}_1) = n + x \\
			-\log{|\bm{\Sigma}_2^{-1}|} + \operatorname{Tr}(\bm{\Sigma}_2^{-1}) = n + y \\
		\end{aligned} \right.
	\end{equation*}
    where $x, y \geq 0$ are fixed constants. Then
	\begin{equation*}
		-\log{|\bm{\Sigma}_2^{-1}||\bm{\Sigma}_1|} + \operatorname{Tr}(\bm{\Sigma}_2^{-1} \bm{\Sigma}_1) - n \leq F(x, y)
	\end{equation*}
    and equality holds if and only if $\bm{\Sigma}_1$ and $\bm{\Sigma}_2$ satisfy
	\begin{equation*}
		\left\{\begin{aligned}
			\bm{\Sigma}_1 & = \bm{Q} \operatorname{diag} \left(w_2(x), 1, \dots, 1 \right) \bm{Q}^{\top} \\
			\bm{\Sigma}_2 &= \bm{Q} \operatorname{diag} \left(w_2(y)^{-1}, 1, \dots, 1 \right) \bm{Q}^{\top} \\
		\end{aligned} \right. 
	\end{equation*}
	where $\bm{Q}$ is an arbitrary orthogonal matrix.
\end{lemma}

\begin{IEEEproof}
    Lemma \ref{lem_MaxSigmaFun} is a rephrase of the original inequality  (H.178) in Appendix H of \cite{zhang2023PropertiesKullbackleiblerDivergence}. 
    Here we clarify the conditions when the equality holds as follows.
	\begin{equation*}
		\begin{aligned}
			&-\log{|\bm{\Sigma}_2^{-1}||\bm{\Sigma}_1|} + \operatorname{Tr}(\bm{\Sigma}_2^{-1} \bm{\Sigma}_1) - n \\
			\leq& x + y + w_2(x)w_2(y) - w_2(x) - w_2(y) + 1 \\
			& \text{~(by the original inequality (H.178) in \cite[Appendix H]{zhang2023PropertiesKullbackleiblerDivergence})} \\
			=& \left[ w_2(x) - 1 \right] \left[ w_2(y) - 1 \right] + x + y \\
			=& F(x, y) \text{~(by the definition of $F(x, y)$)}
		\end{aligned}
	\end{equation*}
	where equality holds if and only if the positive definite matrices $\bm{\Sigma}_1$ and $\bm{\Sigma}_2^{-1}$ have the following eigenvalue decompositions 
	\begin{equation*}
		\left\{\begin{aligned}
			\bm{\Sigma}_1 = \bm{Q} \operatorname{diag} \left(w_2(x), 1, \dots, 1 \right) \bm{Q}^{\top} \\
			\bm{\Sigma}_2^{-1} = \bm{Q} \operatorname{diag} \left(w_2(y), 1, \dots, 1 \right) \bm{Q}^{\top} \\
		\end{aligned} \right.
	\end{equation*}
	  Consequently,
	\begin{equation*}
		\begin{aligned}
			\bm{\Sigma}_2 =& (\bm{\Sigma}_2^{-1})^{-1} = \left( \bm{Q} \operatorname{diag} \left(w_2(y), 1, \dots, 1 \right) \bm{Q}^{\top} \right)^{-1} \\
			=& (\bm{Q}^{\top})^{-1} \operatorname{diag} \left(w_2(y), 1, \dots, 1 \right)^{-1} \bm{Q}^{-1} \\
			=& \bm{Q} \operatorname{diag} \left( w_2(y)^{-1}, 1, \dots, 1 \right) \bm{Q}^{\top}
		\end{aligned} 
	\end{equation*}
	which completes the proof.
\end{IEEEproof}

In particular, the proof  in \cite[Appendix H]{zhang2023PropertiesKullbackleiblerDivergence} relies on a core Lemma, the original Lemma G.5 in \cite[Appendix G]{zhang2023PropertiesKullbackleiblerDivergence}. 
In the following, we provide a significantly more concise proof for Lemma~\ref{lem_MaxF}, which is a strengthened version of the original Lemma~G.5 in \cite[Appendix G]{zhang2023PropertiesKullbackleiblerDivergence}. 
We prove the following Lemma~\ref{lem_IeqOfW2(x/2)} in Subsection~\ref{subsection_IeqOfW2(x/2)} first, and then prove Lemma~\ref{lem_MaxF} in Subsection~\ref{subsection_MaxF}.

\subsection{Lemma~\ref{lem_IeqOfW2(x/2)} and Its Proof} \label{subsection_IeqOfW2(x/2)}
\begin{lemma} \label{lem_IeqOfW2(x/2)}
	For all $x \geq 0$, it holds that
	\[
	w_2(x) - 1 \geq \sqrt{2}\left[ w_2\!\left(\frac{x}{2}\right) - 1 \right]
	\]
	with equality if and only if $x = 0$.
\end{lemma}

\begin{IEEEproof}
	For any $\Delta > 0$, define
	\begin{equation*}
		k(x) = \left[ w_2(x) - 1 \right]^{2} + \left[ w_2(\Delta - x) - 1 \right]^{2}, \quad x \in \left[0, \frac{\Delta}{2} \right]
	\end{equation*}
	Since for all $x > 0$,
	\begin{equation*}
		w_{2}^{\prime}(x) = \frac{ w_{2}(x) }{ w_{2}(x) - 1}
	\end{equation*}
	it follows that
	\begin{equation*}
		\begin{aligned}
			k^{\prime}(x) =&\; 2 \left[ w_2(x) - 1 \right] \frac{ w_{2}(x) }{ w_{2}(x) - 1} \\
			&+ 2 \left[ w_2(\Delta - x) - 1 \right] \frac{ -w_{2}(\Delta - x) }{ w_{2}(\Delta - x) - 1} \\
			=&\; 2\left[ w_2(x) - w_{2}(\Delta - x) \right]
		\end{aligned}
	\end{equation*}
	Because $w_2(x), x \geq 0$ is strictly increasing in $x$, we have $k^{\prime}(x) = 2\left[ w_2(x) - w_{2}(\Delta - x) \right] < 0$ for all $x \in \left(0, \frac{\Delta}{2} \right)$. Hence, $k(0) > k\!\left(\frac{\Delta}{2}\right)$, which yields
	\begin{equation*}
		\left[ w_2(\Delta) - 1 \right]^{2} > 2 \left[ w_2\!\left(\frac{\Delta}{2}\right) - 1 \right]^{2}
	\end{equation*}
	Moreover, since $w_2(x) \geq w_2(0) = 1$ for all $x \geq 0$, it follows from the above that we have:
	\begin{equation*}
		w_2(x) - 1 > \sqrt{2} \left[ w_2\!\left(\frac{x}{2}\right) - 1 \right]
	\end{equation*}
	for all $x > 0$. Moreover, when $x = 0$, both sides equal zero, i.e.,
	\[
	w_2(0) - 1 = \sqrt{2} \left[ w_2(0) - 1 \right] = 0
	\]
	Thus, the lemma is proved.
\end{IEEEproof}

\subsection{Lemma~\ref{lem_MaxF} and Its Proof} \label{subsection_MaxF}

\begin{lemma} \label{lem_MaxF}
	For any $x_1, x_2 \geq 0$ and $y_1, y_2 \geq 0$, the following inequality holds
	\begin{equation*}
		F(x_1,y_1) + F(x_2,y_2) \leq F(x_1 + x_2, y_1 + y_2)
	\end{equation*}
	with equality if and only if $x_1 = y_1 = 0$ or $x_2 = y_2 = 0$.
\end{lemma}

\begin{remark}
	By the identities $f(w_2(x)w_2(y)) = F(x,y) + 1$, which follows directly from the derivation
	\begin{equation*}
		\begin{aligned}
			&f(w_2(x)w_2(y)) \\
			=& w_2(x)w_2(y) - \log{w_2(x)w_2(y)} \\ 
			=& w_2(x)w_2(y) - \log{w_2(x)} - \log{w_2(y)} \\
			=& w_2(x)w_2(y) - \left[ w_2(x) - (x + 1) \right] - \left[ w_2(y) - (y + 1) \right]  \\
			=& \left[ w_2(x) - 1 \right] \left[ w_2(y) - 1 \right] + x + y + 1 \\
			=& F(x,y) + 1
		\end{aligned}
	\end{equation*} it can be observed that Lemma~\ref{lem_MaxF} is a strengthened version of the original Lemma~G.5 in \cite[Appendix G]{zhang2023PropertiesKullbackleiblerDivergence}, whereas the original lemma assumes $x_1 \geq x_2 \geq 0$ and $y_1 \geq y_2 \geq 0$, our result requires only $x_1, x_2 \geq 0$ and $y_1, y_2 \geq 0$. 
\end{remark}

\begin{IEEEproof}
	Without loss of generality, assume $x_1 + x_2 = \Delta_1$ and $y_1 + y_2 = \Delta_2$. 
	If $\Delta_1 = 0$, then $x_1 = x_2 = 0$, and we have
	\begin{equation*}
		\left\lbrace 
		\begin{aligned}
			F(x_1,y_1) + F(x_2,y_2) =&\; 0 + 0 = 0 \\
			F(x_1 + x_2, y_1 + y_2) =&\; F(0, y_1 + y_2) = 0 \\
		\end{aligned}
		\right. 
	\end{equation*}
	Thus, Lemma~\ref{lem_MaxF} holds trivially. 
	A similar argument applies when $\Delta_2 = 0$.
	
	Now consider the case where $\Delta_1 > 0$ and $\Delta_2 > 0$.
	Define
	\begin{equation*}
		K(x, y) = F(x,y) + F(\Delta_1 - x, \Delta_2 - y)
	\end{equation*}
	where $(x, y) \in \hat{\Omega} = \hat{\Omega}(\Delta_1, \Delta_2) = [0,  \Delta_1] \times [0, \Delta_2]$.
	Since
	\begin{equation*}
		\begin{aligned}
			K(x, y) =& \left[ w_2(x) - 1 \right] \left[ w_2(y) - 1 \right] + \Delta_1 + \Delta_2 \\
			&+ \left[ w_2(\Delta_1 - x) - 1 \right] \left[ w_2(\Delta_2 - y) - 1 \right]
		\end{aligned}
	\end{equation*}
	is continuously differentiable on $\operatorname{int}\left( \hat{\Omega} \right) = (0,  \Delta_1) \times (0, \Delta_2)$, a necessary condition for $K(x, y)$ to attain a maximum at $(x, y) \in \operatorname{int}(\Omega)$ is 
	\begin{equation*}
		\left\lbrace 
		\begin{aligned}
			\frac{\partial K(x, y)}{\partial x} = 0 \\
			\frac{\partial K(x, y)}{\partial y} = 0
		\end{aligned}
		\right. 
	\end{equation*}
	Meanwhile, for any $(x, y) \in \operatorname{int}\left( \hat{\Omega} \right)$, we have:
	\begin{equation} \label{eq_partKx}
		\begin{aligned}
			\frac{\partial K(x, y) }{\partial x} =& \frac{ w_{2}(x) }{ w_{2}(x) - 1} \left[ w_2(y) - 1 \right] \\
			&- \frac{ w_{2}(\Delta_1 - x) }{ w_{2}(\Delta_1 - x) - 1}   \left[ w_2(\Delta_2 - y) - 1 \right] \\
		\end{aligned}
	\end{equation}
	Setting $\frac{\partial K(x, y) }{\partial x} = 0$, we obtain:
	\begin{equation} \label{eq_partKxEq0}
		\begin{aligned}
			\frac{ w_{2}(x) }{ w_{2}(x) - 1} \left[ w_2(y) - 1 \right] = \frac{ w_{2}(\Delta_1 - x) }{ w_{2}(\Delta_1 - x) - 1}   \left[ w_2(\Delta_2 - y) - 1 \right]
		\end{aligned}
	\end{equation}
	Similarly, setting $\frac{\partial K(x, y) }{\partial y} = 0$, we obtain:
	\begin{equation} \label{eq_partKyEq0}
		\begin{aligned}
			\left[ w_2(x) - 1 \right] \frac{ w_{2}(y) }{ w_{2}(y) - 1} = \left[ w_2(\Delta_1 - x) - 1 \right] \frac{ w_{2}(\Delta_2 - y) }{ w_{2}(\Delta_2 - y) - 1}
		\end{aligned}
	\end{equation}
	From~\eqref{eq_partKxEq0} and~\eqref{eq_partKyEq0}, it follows that
	\begin{equation} \label{eq_partKxyEq0}
		\begin{aligned}
			w_2(x) w_{2}(y) = w_2(\Delta_1 - x) w_{2}(\Delta_2 - y)
		\end{aligned}
	\end{equation}
	Substituting~\eqref{eq_partKxyEq0} into~\eqref{eq_partKx} yields
	\begin{equation*} 
		\begin{aligned}
			&\frac{\partial K(x, y) }{\partial x} \\
			=& \frac{ w_2(x) w_{2}(y) - w_{2}(x) }{ w_{2}(x) - 1} - \frac{ w_2(x) w_{2}(y) - w_{2}(\Delta_1 - x) }{ w_{2}(\Delta_1 - x) - 1}  \\
			=& \frac{w_2(x) w_{2}(y) \left[ w_{2}(\Delta_1 - x) - w_{2}(x) \right] + w_{2}(x) - w_{2}(\Delta_1 - x) }{\left[ w_{2}(x) - 1 \right] \left[ w_{2}(\Delta_1 - x) - 1 \right]} \\
			=& \frac{w_2(x) w_{2}(y) - 1 }{\left[ w_{2}(x) - 1 \right] \left[ w_{2}(\Delta_1 - x) - 1 \right]} \left[ w_{2}(\Delta_1 - x) - w_{2}(x) \right]
		\end{aligned}
	\end{equation*}
	Since $\Delta_1 > x > 0$ and $y > 0$, we have $w_2(x), w_{2}(\Delta_1 - x), w_2(y) > 1$, and therefore
	\begin{equation*} 
		\frac{w_2(x) w_{2}(y) - 1 }{\left[ w_{2}(x) - 1 \right] \left[ w_{2}(\Delta_1 - x) - 1 \right]} > 0
	\end{equation*}
	Consequently, $\frac{\partial K(x, y) }{\partial x} = 0$ holds if and only if $x = \frac{\Delta_1}{2}$.
	Similarly, we obtain
	\begin{equation*} 
		\begin{aligned}
			&\frac{\partial K(x, y) }{\partial y} \\
			=& \frac{w_2(x) w_{2}(y) - 1 }{\left[ w_{2}(y) - 1 \right] \left[ w_{2}(\Delta_2 - y) - 1 \right]} \left[ w_{2}(\Delta_2 - y) - w_{2}(y) \right]
		\end{aligned}
	\end{equation*}
	and $\frac{\partial K(x, y) }{\partial y} = 0$ holds if and only if $y = \frac{\Delta_2}{2}$.
	
	Therefore, the maximum of $K(x, y)$ can only be attained at $(x, y) \in \left\{ \left( \frac{\Delta_1}{2}, \frac{\Delta_2}{2} \right) \right\} \cup \partial \hat{\Omega}$, where $\partial \hat{\Omega} = \hat{\Omega} \setminus \operatorname{int}\left( \hat{\Omega} \right)$ denotes the boundary of $\hat{\Omega}$, which can be decomposed as
	\begin{equation*}
		\partial \hat{\Omega} = \partial \hat{\Omega}_{L} \cup \partial \hat{\Omega}_{U} \cup \partial \hat{\Omega}_{R} \cup \partial \hat{\Omega}_{D}
	\end{equation*}
	and
	\begin{equation*}
		\left\lbrace 
		\begin{aligned}
			\partial \hat{\Omega}_{L} &= \{0\} \times [0, \Delta_2] \\
			\partial \hat{\Omega}_{U} &= [0, \Delta_1] \times \{\Delta_2\} \\
			\partial \hat{\Omega}_{R} &= \{\Delta_1\} \times [0, \Delta_2] \\
			\partial \hat{\Omega}_{D} &= [0, \Delta_1] \times \{0\}
		\end{aligned}
		\right. 
	\end{equation*}
	
	On $\partial \hat{\Omega}_{L}$, we have
	\begin{equation*}
		\begin{aligned}
			\left. K(x, y) \right|_{\partial \hat{\Omega}_{L}} =& K(0, y) \\
			=& \Delta_1 + \Delta_2 + \left[ w_2(\Delta_1) - 1 \right] \left[ w_2(\Delta_2 - y) - 1 \right]
		\end{aligned}
	\end{equation*}
	which is decreasing in $y$. Hence, $\left. K(x, y) \right|_{\partial \hat{\Omega}_{L}} \leq K(0, 0)$. Similarly, $\left. K(x, y) \right|_{\partial \hat{\Omega}_{D}} \leq K(0, 0)$.
	Furthermore, $K(x, y)$ satisfies the symmetry
	\begin{equation*}
		K(x, y) = K(\Delta_1 - x, \Delta_2 - y)
	\end{equation*}
	for all $(x, y) \in \hat{\Omega}$. Therefore, $\left. K(x, y) \right|_{\partial \hat{\Omega}_{R}} \leq K(\Delta_1, \Delta_2)$ and $\left. K(x, y) \right|_{\partial \hat{\Omega}_{U}} \leq K(\Delta_1, \Delta_2)$.
	
	Direct computation shows that 
	\begin{equation*}
		K(0, 0) = K(\Delta_1, \Delta_2) = F(\Delta_1, \Delta_2)
	\end{equation*}
	Moreover, by Lemma~\ref{lem_MaxF}, we have
	\begin{equation*}
		\begin{aligned}
			K(0, 0) =& \Delta_1 + \Delta_2 + \left[ w_2(\Delta_1) - 1 \right] \left[ w_2(\Delta_2) - 1 \right] \\
			>& \Delta_1 + \Delta_2 + \sqrt{2}\left[ w_2(\Delta_1 / 2) - 1 \right] \sqrt{2}\left[ w_2(\Delta_2 / 2) - 1 \right] \\
			=& K(\Delta_1 / 2, \Delta_2 / 2)
		\end{aligned}
	\end{equation*}
	Consequently, for all $\Delta_1, \Delta_2 > 0$, we have $K(x, y) \leq K(0, 0) = F(\Delta_1, \Delta_2)$ for all $(x, y) \in \hat{\Omega}$, with equality if and only if $(x, y) = (0, 0)$ or $(x, y) = (\Delta_1, \Delta_2)$. This establishes Lemma~\ref{lem_MaxF} for the case $\Delta_1, \Delta_2 > 0$.
	
	The lemma is thus proved.
	
\end{IEEEproof}

\section{Lemma~\ref{lem_MaxH} and Its Proof} \label{appendix_MaxH}
\noindent
\begin{lemma} \label{lem_MaxH}
	For all $\Delta_1 > 0$ and $\Delta_2 > 0$, the function $H(x, y; \Delta_1, \Delta_2)$ with $(x, y) \in \Omega(\Delta_1, \Delta_2)$ satisfies
	\begin{equation*}
		H(x, y; \Delta_1, \Delta_2) \leq \frac{1}{2} F(2\Delta_1, 2\Delta_2)
	\end{equation*}
    with equality if and only if $(x, y) = (2\Delta_1, 2\Delta_2)$.
\end{lemma}
To prove Lemma~\ref{lem_MaxH}, we need prove the following Lemma~\ref{lem_MaxHNotOnInternal} in Lemma~\ref{subappendix_MaxHNotOnInternal} and Lemma~\ref{lem_W2tGreatSqrt2t} in Lemma~\ref{subsection_W2tGreatSqrt2t} first. Then we prove Lemma~\ref{lem_MaxH} in Lemma~\ref{subappendix_MaxH}.

\subsection{Lemma~\ref{lem_MaxHNotOnInternal} and Its Proof} \label{subappendix_MaxHNotOnInternal}

\begin{lemma} \label{lem_MaxHNotOnInternal}
	For all  $\Delta_1, \Delta_2 > 0$, $H(x, y; \Delta_1, \Delta_2)$ cannot attain a critical point in $\operatorname{int}\left( \Omega(\Delta_1, \Delta_2) \right)$. 
\end{lemma}

\begin{IEEEproof}
    We only need prove that following equations has no solution in $\operatorname{int}\left( \Omega(\Delta_1, \Delta_2) \right)$.
    \begin{equation*}
		\left\lbrace 
		\begin{aligned}
			\frac{\partial H(x, y; \Delta_1, \Delta_2)}{\partial x} = 0 \\
			\frac{\partial H(x, y; \Delta_1, \Delta_2)}{\partial y} = 0
		\end{aligned}
		\right. 
	\end{equation*}
    
	First, we have
	\begin{align*}
		& H(x, y; \Delta_1, \Delta_2) \\
		=& \frac{1}{2} \left[ F(x, y) + G(x, y; \Delta_1, \Delta_2) \right] \\
		=& \left[ w_2(x) - 1 \right] \left[ w_2(y) - 1 \right] + x + y + w_2(y) \, (2\Delta_1 - x) \\
		& + (2\Delta_2 - y) + 2 \sqrt{w_2(y) \, (2\Delta_2 - y) \, (2\Delta_1 - x) } \\
		=& \left[ w_2(x) - 1 \right] \left[ w_2(y) - 1 \right] + x + w_2(y) \, (2\Delta_1 - x) \\
		& + 2\Delta_2 + 2 \sqrt{w_2(y) \, (2\Delta_2 - y) \, (2\Delta_1 - x) }
	\end{align*}
	Since for all $x > 0$,
	\begin{equation*}
		w_{2}^{\prime}(x) = \frac{ w_{2}(x) }{ w_{2}(x) - 1}
	\end{equation*}
	then for all $(x,y) \in \operatorname{int}\left( \Omega(\Delta_1, \Delta_2) \right)$, it hold that 
	\begin{equation*}
		\begin{aligned}
			& \frac{\partial H(x, y; \Delta_1, \Delta_2) }{\partial x} \\
			=& \frac{ w_{2}(x) }{ w_{2}(x) - 1} \left[ w_2(y) - 1 \right] + 1 - w_2(y) - \frac{\sqrt{w_2(y) \, (2\Delta_2 - y)} }{ 2\Delta_1 - x } \\
			=& \frac{w_2(y) - 1}{w_2(x) - 1} - \frac{\sqrt{w_2(y) \, (2\Delta_2 - y)} }{\sqrt{2\Delta_1 - x}}
		\end{aligned}
	\end{equation*}
	and
	\begin{equation*}
		\begin{aligned}
			& \frac{\partial H(x, y; \Delta_1, \Delta_2) }{\partial y} \\
			=& \left[ w_2(x) - 1 \right] \frac{ w_{2}(y) }{ w_{2}(y) - 1} + \frac{ w_{2}(y) }{ w_{2}(y) - 1} (2\Delta_1 - x) \\
			& + \frac{\sqrt{2\Delta_1 - x}}{\sqrt{w_2(y) \, (2\Delta_2 - y)}} \left[ \frac{ w_{2}(y) }{ w_{2}(y) - 1} (2\Delta_2 - y) - w_2(y) \right] \\
			=& \frac{w_{2}(y)}{w_{2}(y) - 1}\left[ w_2(x) - 1 + (2\Delta_1 - x)\right] \\
			&- \frac{\sqrt{2\Delta_1 - x}}{\sqrt{w_2(y) \, (2\Delta_2 - y)}} \frac{w_{2}(y)}{w_{2}(y) - 1} \left[ w_2(y) - 1 - (2\Delta_2 - y)\right]
		\end{aligned}
	\end{equation*}
    
	Setting $\frac{\partial H(x, y; \Delta_1, \Delta_2) }{\partial x} = 0$ yields
	\begin{equation} \label{eq_partHpartx}
		\frac{\sqrt{2\Delta_1 - x}}{w_2(x) - 1} = \frac{\sqrt{w_2(y) \, (2\Delta_2 - y)}}{w_2(y) - 1}
	\end{equation}
	Setting $\frac{\partial H(x, y; \Delta_1, \Delta_2) }{\partial y} = 0$ yields
	\begin{equation} \label{eq_partHparty}
		\frac{\sqrt{2\Delta_1 - x}}{w_2(x) - 1 + (2\Delta_1 - x)} = \frac{\sqrt{w_2(y) \, (2\Delta_2 - y)}}{w_2(y) - 1 - (2\Delta_2 - y)}
	\end{equation}
	Combining \eqref{eq_partHpartx} and \eqref{eq_partHparty} gives
	\begin{equation*}
		\frac{w_2(x) - 1}{w_2(x) - 1 + (2\Delta_1 - x)} = \frac{w_2(y) - 1}{w_2(y) - 1 - (2\Delta_2 - y)}
	\end{equation*}
	which implies
	\begin{equation*}
		\frac{1}{1 +\dfrac{2\Delta_1 - x}{w_2(x) - 1}} = \frac{1}{1 + \dfrac{-(2\Delta_2 - y)}{w_2(y) - 1}}
	\end{equation*}
	Since $x \in (0, 2\Delta_1)$, it holds that $\dfrac{2\Delta_1 - x}{w_2(x) - 1} > 0$; meanwhile, since $y \in (0, 2\Delta_2)$, we have $\dfrac{-(2\Delta_2 - y)}{w_2(y) - 1} < 0$. Hence, the above equality cannot hold, and there is no critical point in $\operatorname{int}\left( \Omega(\Delta_1, \Delta_2) \right)$. The lemma is proved.
\end{IEEEproof}

\subsection{Lemma~\ref{lem_W2tGreatSqrt2t} and Its Proof} \label{subsection_W2tGreatSqrt2t}

\begin{lemma} \label{lem_W2tGreatSqrt2t}
	$w_2(t) - 1 > \sqrt{2t}$ holds for all $t > 0$.
\end{lemma}

\begin{IEEEproof}
	Let $x = w_2(t)$. Since $t > 0$, it follows from the definition of $w_2(t)$ that $x > 1$ and
	\begin{equation*}
		t = x - \ln x - 1
	\end{equation*}
	Then we have: 
	\begin{equation*}
		\begin{aligned}
			& w_2(t) - 1 > \sqrt{2t} ~ \text{holds for all } t > 0 \\
			\iff\ & x - 1 > \sqrt{2(x - \ln x - 1)} ~ \text{holds for all } x > 1 \\
			\iff\ & \frac{1}{2}(x - 1)^2 - (x - \ln x - 1) > 0 ~ \text{holds for all } x > 1
		\end{aligned}
	\end{equation*}
	Define $r(x) = \frac{1}{2}(x - 1)^2 - (x - \ln x - 1)$ for $x > 0$. Then
	\begin{equation*}
		r'(x) = x - 1 - \left(1 - \frac{1}{x}\right) = x + \frac{1}{x} - 2
	\end{equation*}
	Since $x + \frac{1}{x} - 2 = \frac{(x - 1)^2}{x} > 0$ for all $x > 1$, the function $r(x)$ is strictly increasing on $(1, \infty)$. Consequently, for all $x > 1$ we have $r(x) > r(1) = 0$, which establishes the desired inequality. The lemma is proved.
\end{IEEEproof}

\subsection{Proof of Lemma~\ref{lem_MaxH}} \label{subappendix_MaxH}

\begin{IEEEproof}
	For notational convenience we henceforth write $\Omega$ in place of $\Omega(\Delta_1, \Delta_2)$

    Since $H(x, y; \Delta_1, \Delta_2)$ is a continuous function on the compact set $\Omega$, it attains its maximum at some point $(x^*, y^*) \in \Omega$. We denote the maximum value by
    \begin{equation*}
	   H^*(\Delta_1, \Delta_2) = \max_{(x,y) \in \Omega} H(x, y; \Delta_1, \Delta_2)
    \end{equation*}

    Meanwhile, since $H(x, y; \Delta_1, \Delta_2)$ is continuously differentiable in $\operatorname{int}\left( \Omega\right)$, a necessary condition for a point $(x, y) \in \operatorname{int}\left( \Omega\right)$ to be a local maximum is that $(x, y)$ is a critical point. By Lemma~\ref{lem_MaxHNotOnInternal}, $H(x, y; \Delta_1, \Delta_2)$ cannot attain a maximum on $\operatorname{int}\left( \Omega\right)$. That is,
    \begin{equation} \label{ieq_HInt<Bound}
	   H(x, y; \Delta_1, \Delta_2) < H^*(\Delta_1, \Delta_2), \text{for all} (x, y) \in \operatorname{int}\left( \Omega \right)
    \end{equation}
    
    Then the maximum of $H$ must be attained on the boundary $\partial \Omega = \Omega \setminus \operatorname{int}(\Omega)$, which satisfies
    \begin{equation*}
	   \partial \Omega = \partial \Omega_{L} \cup \partial \Omega_{U} \cup \partial \Omega_{R} \cup \partial \Omega_{D} \cup \{(2\Delta_1, 2\Delta_2)\}
    \end{equation*}
    where 
\begin{equation*}
	\left\lbrace 
	\begin{aligned}
		\partial \Omega_{L} &= \{0\} \times (0, 2\Delta_2) \\
		\partial \Omega_{U} &= [0, 2\Delta_1) \times \{2\Delta_2\} \\
		\partial \Omega_{R} &= \{2\Delta_1\} \times [0, 2\Delta_2) \\
		\partial \Omega_{D} &= [0, 2\Delta_1) \times \{0\}
	\end{aligned}
	\right.
\end{equation*}
	
	Consider $(x, y) \in \partial \Omega_{L}$ first.  
	From
	\begin{equation*}
		\begin{aligned}
			& \frac{\partial H(x, y; \Delta_1, \Delta_2)}{\partial x} \bigg|_{(x, y) = (0, y_0)} \\
			=& \frac{ w_{2}(x) }{ w_{2}(x) - 1} \left[ w_2(y) - 1 \right] + 1 - w_2(y) - \frac{\sqrt{w_2(y) \, (2\Delta_2 - y)} }{ 2\Delta_1 - x }
		\end{aligned}
	\end{equation*}
	it follows that for any $y_0 \in (0, 2\Delta_2)$,
	\begin{equation*}
		\lim_{x \rightarrow 0^{+}} \frac{\partial H(x, y; \Delta_1, \Delta_2)}{\partial x} = +\infty
	\end{equation*}
	Hence, for each $(x, y) \in \{0\} \times (0, 2\Delta_2)$, there exists $x_0 \in (0, 2\Delta_1)$ such that
	\begin{equation*}
		\begin{aligned}
			H(0, y; \Delta_1, \Delta_2) &< H(x_0, y; \Delta_1, \Delta_2) \\
			&< H^*(\Delta_1, \Delta_2) \quad \text{(by~\eqref{ieq_HInt<Bound})}
		\end{aligned}
	\end{equation*}
	Therefore, the maximum of $H$ cannot be attained on $\partial \Omega_{L}$.
	
	Next, consider $(x, y) \in \partial \Omega_{U}$.  
	For such points,
	\begin{equation*}
		\begin{aligned}
			H(x, 2\Delta_2; \Delta_1, \Delta_2) =& \left[ w_2(x) - 1 \right] \left[ w_2(2\Delta_2) - 1 \right] \\
			& + x + w_2(2\Delta_2) \, (2\Delta_1 - x) + 2\Delta_2
		\end{aligned}
	\end{equation*}
	Differentiating with respect to $x$ yields
	\begin{equation*}
		\begin{aligned}
			&\frac{d H(x, 2\Delta_2; \Delta_1, \Delta_2)}{dx} \\
			=& \left[ w_2(2\Delta_2) - 1 \right] \frac{w_2(x)}{w_2(x) - 1} + 1 - w_2(2\Delta_2) \\
			=& \frac{w_2(2\Delta_2) - 1}{w_2(x) - 1} > 0
		\end{aligned}
	\end{equation*}
	Thus, $H(x, 2\Delta_2; \Delta_1, \Delta_2) < H(2\Delta_1, 2\Delta_2; \Delta_1, \Delta_2)$ for all $(x, y) \in \partial \Omega_{U}$,  
	so the maximum cannot be attained on $\partial \Omega_{U}$.
	
	Now let $(x, y) \in \partial \Omega_{R}$.  
	Then
	\begin{equation*}
		\begin{aligned}
			&H(2\Delta_1, y; \Delta_1, \Delta_2) \\
			=& \left[ w_2(2\Delta_1) - 1 \right] \left[ w_2(y) - 1 \right] + 2\Delta_1 + 2\Delta_2
		\end{aligned}
	\end{equation*}
	which is strictly increasing in $y$. Consequently,
	\begin{equation*}
		H(2\Delta_1, y; \Delta_1, \Delta_2) < H(2\Delta_1, 2\Delta_2; \Delta_1, \Delta_2)
	\end{equation*}
	with equality only at $y = 2\Delta_2$. Hence, the maximum is not attained on $\partial \Omega_{R}$.
	
	Finally, consider $(x, y) \in \partial \Omega_{D}$.  
	In this case,
	\begin{equation*}
		\begin{aligned}
			H(x, 0; \Delta_1, \Delta_2) =& 2\Delta_1 + 2\Delta_2 + 2 \sqrt{2\Delta_2 \, (2\Delta_1 - x) }
		\end{aligned}
	\end{equation*}
	which is strictly decreasing in $x$. Therefore,
	\begin{equation*}
		H(x, 0; \Delta_1, \Delta_2) \leq H(0, 0; \Delta_1, \Delta_2)
	\end{equation*}
	Moreover,
	\begin{equation*}
		\begin{aligned}
			&H(0, 0; \Delta_1, \Delta_2) \\
			=& 2\Delta_1 + 2\Delta_2 + 2 \sqrt{ 2\Delta_2 \, 2\Delta_1} \\
			=& 2\Delta_1 + 2\Delta_2 + \sqrt{2 \cdot (2\Delta_2)} \, \sqrt{2 \cdot (2\Delta_1)} \\
			<& 2\Delta_1 + 2\Delta_2 + \left[ w_2(2\Delta_1) - 1 \right] \left[ w_2(2\Delta_2) - 1 \right] \, \text{(by Lemma~\ref{lem_W2tGreatSqrt2t})} \\
			=& H(2\Delta_1, 2\Delta_2; \Delta_1, \Delta_2)
		\end{aligned}
	\end{equation*}
	Thus, the maximum cannot be attained on $\partial \Omega_{D}$.
	
	The only remaining candidate is the corner point $(2\Delta_1, 2\Delta_2)$.  
	Hence, $H$ attains its maximum precisely at this point, and we conclude that
	\begin{equation*}
		H^*(\Delta_1, \Delta_2) = H(2\Delta_1, 2\Delta_2; \Delta_1, \Delta_2) = \frac{1}{2} F(2\Delta_1, 2\Delta_2)
	\end{equation*}
	with equality if and only if $(x, y) = (2\Delta_1, 2\Delta_2)$.  
	This completes the proof.
\end{IEEEproof}

\section{Proof of Lemma~\ref{lem_KLSupremumForNormal}} \label{appendix_KLSupremumForNormal}

\begin{IEEEproof}
	Applying \eqref{eq_closedFormKLForGauss}, the condition $\KL\left(\N{N}_1 \,\|\, \N{N}(\vect{0}, \bm{I})\right) = \Delta_1 $ is equivalent to:
	\begin{equation} \label{eq_KL1Restriction}
		-\log|\bm{\Sigma}_1| + \operatorname{Tr}(\bm{\Sigma}_1) - n + \vect{\mu}_1^\top \vect{\mu}_1 = 2 \Delta_1 
	\end{equation}
	By the original equations (C.25) and (C.28) in \cite[Appendix C]{zhang2023PropertiesKullbackleiblerDivergence}, we have $-\log|\bm{\Sigma}_1| + \operatorname{Tr}(\bm{\Sigma}_1) - n \geq 0$. Combined with $\vect{\mu}_1^\top \vect{\mu}_1 \geq 0$, we introduce an auxiliary variable $x = -\log|\bm{\Sigma}_1| + \operatorname{Tr}(\bm{\Sigma}_1) - n \in [0, 2\Delta_1]$ such that constraint~\eqref{eq_KL1Restriction} is equivalent to:
	\begin{equation*}
		\left\{\begin{aligned}
			-\log|\bm{\Sigma}_1| + \operatorname{Tr}(\bm{\Sigma}_1)  &= n + x \\
			\vect{\mu}_1^\top \vect{\mu}_1 &= 2 \Delta_1 - x 
		\end{aligned} \right.
	\end{equation*}
	Similarly, we have that the condition $\KL\left(\N{N}(\vect{0}, \bm{I}) \,\|\, \N{N}_2 \right) = \Delta_2 $ is equivalent to:
	\begin{equation*}
		\left\{\begin{aligned}
			-\log|\bm{\Sigma}_2 ^ {-1}| + \operatorname{Tr}(\bm{\Sigma}_2 ^ {-1} )  &= n + y \\
			\vect{\mu}_2^\top \bm{\Sigma}_2 ^ {-1} \vect{\mu}_2 &= 2 \Delta_2 - y 
		\end{aligned} \right.
	\end{equation*}
	where the introduced auxiliary variable $y \in [0, 2\Delta_2]$.
	
	Combining~\eqref{eq_closedFormKLForGauss}, maximizing $\KL\left(\N{N}_1 \,\|\, \N{N}_2 \right)$ is equivalent to solving the optimization Problem~\ref{question_normalgaussTrans} stated in the main text.
	
	For any fixed $(x_{0},y_{0}) \in \Omega(\Delta_1, \Delta_2)$,  Problem~\ref{question_normalgaussTrans} can be decomposed into two subproblems coupled only through the decision variables $\bm{\Sigma}_2$, one involving $\vect{\mu_1}$, $\vect{\mu_2}$, and $\bm{\Sigma}_2$, denoted as Problem~\ref{problem_mu}, and the other involving $\bm{\Sigma}_1$ and $\bm{\Sigma}_2$, denoted as Problem~\ref{problem_sigma}.
	
	For the Problem~\ref{problem_mu}, applying Lemma~\ref{lem_MaxMuFun} we have : 
	\begin{equation} \label{ieq_Mu}
		\begin{aligned}
			&(\vect{\mu}_2 - \vect{\mu}_1)^\top \bm{\Sigma}_2^{-1} (\vect{\mu}_2 - \vect{\mu}_1) \\
			\leq& \left( \sqrt{w_2(y_{0}) (2 \Delta_1 - x_{0})} + \sqrt{2 \Delta_2 - y_{0}} \right)^2 \\
			=& G(x_{0},y_{0}; \Delta_1, \Delta_2) \text{~(By the definition of $G(x,y; \Delta_1, \Delta_2)$)}
		\end{aligned}
	\end{equation}
	Equality holds if
	\begin{equation*}
		\left\{ 
		\begin{aligned}
			\vect{\mu}_1 &= x_{1} \vect{e}_{1} \\
			\vect{\mu}_2 &= y_{1} \vect{e}_{1} \\
			\bm{\Sigma}_2^{-1} &= \bm{Q} \operatorname{diag} \left(w_2(y_{0}), 1, \dots, 1 \right) \bm{Q}^{\top}
		\end{aligned} 
		\right.
	\end{equation*}
	\textit{i.e.}
	\begin{equation*}
		\left\{ 
		\begin{aligned}
			\vect{\mu}_1 &= x_{1} \vect{e}_{1} \\
			\vect{\mu}_2 &= y_{1} \vect{e}_{1} \\
			\bm{\Sigma}_2 &= \bm{Q} \operatorname{diag} \left(w_2(y_{0})^{-1}, 1, \dots, 1 \right) \bm{Q}^{\top}
		\end{aligned} 
		\right.
	\end{equation*}
	where $(x_1, y_1) = \pm \left(  \sqrt{\varepsilon_1}, - \sqrt{\frac{\varepsilon_2}{w_2(y_{0})}} \right)$, $\bm{Q}$ is an orthogonal matrix and $\vect{e}_{1}$ is the first column of $\bm{Q}$.
    
	For the Problem~\ref{problem_sigma}, applying Lemma~\ref{lem_MaxSigmaFun} we have : 
	\begin{equation} \label{ieq_Sigma}
		-\log{|\bm{\Sigma}_2^{-1}||\bm{\Sigma}_1|} + \operatorname{Tr}(\bm{\Sigma}_2 ^ {-1} \bm{\Sigma}_1) - n \leq F(x_{0}, y_{0})
	\end{equation}
	Equality holds if and only if
	\begin{equation*}
		\left\{
		\begin{aligned}
			\bm{\Sigma}_1 &= \bm{Q} \operatorname{diag}\bigl(w_2(x_{0}), 1, \dots, 1\bigr) \bm{Q}^\top \\
			\bm{\Sigma}_2 &= \bm{Q} \operatorname{diag}\bigl(w_2(y_{0})^{-1}, 1, \dots, 1\bigr) \bm{Q}^\top
		\end{aligned}
		\right.
	\end{equation*}
	where $\bm{Q}$ is an orthogonal matrix.
	
	Since for any $(x_{0},y_{0}) \in \Omega(\Delta_1, \Delta_2)$, both conditions prescribe the same structure for the coupled matrix $\bm{\Sigma}_2$, they are simultaneously satisfiable.
	Consequently,
	\begin{equation*} 
		\begin{aligned}
			&\KL\left(\N{N}_1 \,\|\, \N{N}_2 \right) \\
			=& \frac{1}{2} \Bigl( -\log{|\bm{\Sigma}_2^{-1}||\bm{\Sigma}_1|} + \operatorname{Tr}(\bm{\Sigma}_2 ^ {-1} \bm{\Sigma}_1) - n \\
			& + (\vect{\mu_2} - \vect{\mu_1})^\top \bm{\Sigma}_2 ^ {-1} (\vect{\mu_2} - \vect{\mu_1}) \Bigr) \\
			\leq& \frac{1}{2} \left( F(x_{0},y_{0}) + G(x_{0},y_{0}; \Delta_1, \Delta_2) \right) ~ \text{(by \eqref{ieq_Sigma} and \eqref{ieq_Mu})}\\
			=& H(x_{0},y_{0}; \Delta_1, \Delta_2) \text{~(By the definition of $H(x,y; \Delta_1, \Delta_2)$)} \\
			\leq& \frac{1}{2} F(2\Delta_1, 2\Delta_2) ~ \text{(by Lemma~\ref{lem_MaxH})}\\
		\end{aligned}
	\end{equation*} 
	where the second inequality holds with equality if and only if $(x_{0},y_{0}) = (2\Delta_1, 2\Delta_2)$. In this case,
	\begin{equation*}
		\left\{
		\begin{aligned}
			\vect{\mu}_1^\top \vect{\mu}_1 &= 2 \Delta_1 - x_{0} = 0 \\
			\vect{\mu}_2^\top \bm{\Sigma}_2 ^ {-1} \vect{\mu}_2 &= 2 \Delta_2 - y_{0} = 0
		\end{aligned}
		\right.
		\quad \Leftrightarrow \quad
		\vect{\mu}_1 = \vect{\mu}_2 = \vect{0}
	\end{equation*}
	and consequently
	\begin{equation*}
		(\vect{\mu_2} - \vect{\mu_1})^\top \bm{\Sigma}_2 ^ {-1} (\vect{\mu_2} - \vect{\mu_1}) = G(x_{0},y_{0}; \Delta_1, \Delta_2) = 0
	\end{equation*}
	so equality in \eqref{ieq_Mu} holds automatically; meanwhile, the inequality \eqref{ieq_Sigma} becomes an equality if and only if
	\begin{equation*}
		\left\{
		\begin{aligned}
			\bm{\Sigma}_1 &= \bm{Q} \operatorname{diag}\bigl(w_2(2\Delta_1), 1, \dots, 1\bigr) \bm{Q}^\top \\
			\bm{\Sigma}_2 &= \bm{Q} \operatorname{diag}\bigl(w_2(2\Delta_2)^{-1}, 1, \dots, 1\bigr) \bm{Q}^\top
		\end{aligned}
		\right.
	\end{equation*}
	where $\bm{Q}$ is an orthogonal matrix.
	Thus, equality in the overall bound is achieved if and only if
	\begin{equation*}
		\left\{
		\begin{aligned}
			\vect{\mu}_1 &= \vect{\mu}_2 = \vect{0} \\
			\bm{\Sigma}_1 &= \bm{Q} \operatorname{diag}\bigl(w_2(2\Delta_1), 1, \dots, 1\bigr) \bm{Q}^\top \\
			\bm{\Sigma}_2 &= \bm{Q} \operatorname{diag}\bigl(w_2(2\Delta_2)^{-1}, 1, \dots, 1\bigr) \bm{Q}^\top
		\end{aligned}
		\right.
	\end{equation*}
	where $\bm{Q}$ is an orthogonal matrix.
	This completes the proof.
\end{IEEEproof}

\section{Proof of Theorem~\ref{theorem_KLSupremumForGeneral}} \label{section_KLSupremumForGeneral}

The proof of Theorem~\ref{theorem_KLSupremumForGeneral} is the same as the original Theorem 4 in \cite{zhang2023PropertiesKullbackleiblerDivergence}. Here we present the proof in the following for integrity.
The proof needs the following  Proposition~\ref{proposition_DiffPreservesKLDivergence} \cite{nielsen2020ElementaryIntroductionInformation} and Lemma~\ref{lem_MultivariateNormalDistribution} \cite{chatfield2018IntroductionMultivariateAnalysis}.
\begin{proposition} \cite{nielsen2020ElementaryIntroductionInformation} \label{proposition_DiffPreservesKLDivergence}
	Let $z = f(x)$ be a diffeomorphism, $X_1 \sim p_X$ and $X_2 \sim q_X$ be two random variables and $Z_1 = f(X_1) \sim p_Z$, $Z_2 = f(X_2) \sim q_Z$. Then $KL(p_X \| q_X) = KL(p_Z \| q_Z)$.
\end{proposition}

\begin{lemma}\cite{chatfield2018IntroductionMultivariateAnalysis} \label{lem_MultivariateNormalDistribution}
	Let $\vect{X} \in \mathbb{R}^n$ be a random vector such that
	\[
	\vect{X} \sim \mathcal{N}(\vect{\mu}, \bm{\Sigma})
	\]
	where $\vect{\mu} \in \mathbb{R}^n$ and 
	$\bm{\Sigma} \in \mathbb{R}^{n \times n}$ is a symmetric positive definite covariance matrix.
	The invertible linear transformation $\vect{T}$ is defined as:
	\begin{equation*}
		\vect{T}: \vect{X} \mapsto \vect{X}^{'} = A \vect{X} + \vect{b}
	\end{equation*}
	where $A \in \mathbb{R}^{n \times n}$ is an invertible matrix and $\vect{b} \in \mathbb{R}^n$.
	Then $\vect{X}^{'} = \vect{T}( \vect{X} )$ follows an $n$-dimensional multivariate normal distribution:
	\[
	\vect{X}^{'} \sim \N{N}(A \vect{\mu} + \vect{b},\; A \bm{\Sigma} A^\top)
	\]
	where $A \bm{\Sigma} A^\top$ is a positive definite matrix.
\end{lemma}

\begin{IEEEproof}
	Introduce random vectors $\vect{X}_i \sim \N{N}_i$ for $i \in \{1,2,3\}$. Since $\bm{\Sigma}_2$ is a fixed positive definite matrix, there exists an invertible matrix $\bm{B}_2$ such that:
	\begin{equation*}
		\bm{\Sigma}_2 = \bm{B}_2 \bm{B}_2^{\top}
	\end{equation*}
	and the invertible linear transformation $\vect{T}$:
	\begin{equation*}
		\vect{T}: \vect{X} \mapsto \vect{X}^{'} = \vect{T}(\vect{X}) = \bm{B}_2^{-1} (\vect{X} - \vect{\mu}_2)
	\end{equation*}
	such that \cite{zhang2023PropertiesKullbackleiblerDivergence}:
	\begin{equation*}
		\vect{X}_2^{'} = \vect{T}(\vect{X}_2) \sim \N{N}(\vect{0},~\bm{I})
	\end{equation*}
	Applying transformation $\vect{T}$ to the random vectors $\vect{X}_i \sim \N{N}_i$ for $i \in \{1,3\}$, yielding:
	\begin{equation*}
		\left\{\begin{aligned}
			& \vect{X}_1^{'} \sim \N{N}(\bm{B}_2^{-1} (\vect{\mu}_1 - \vect{\mu}_2),~ \bm{B}_2^{-1} \bm{\Sigma}_1 \bm{B}_2^\top) = \N{N}(\bm{\mu}_1^{'}, \bm{\Sigma}_1^{'}) \\  
			& \vect{X}_3^{'} \sim \N{N}(\bm{B}_2^{-1} (\vect{\mu}_3 - \vect{\mu}_2),~ \bm{B}_2^{-1} \bm{\Sigma}_3 \bm{B}_2^\top) = \N{N}(\bm{\mu}_3^{'}, \bm{\Sigma}_3^{'}) \\ 
		\end{aligned} \right.
	\end{equation*}
	Thus we have:
	\begin{equation} \label{eq_mu'sigma'B}
		\left\{\begin{aligned}
			\bm{\mu}_1^{'} &= \bm{B}_2^{-1} (\vect{\mu}_1 - \vect{\mu}_2) \\ 
			\bm{\Sigma}_1^{'} &= \bm{B}_2^{-1} \bm{\Sigma}_1 \bm{B}_2^\top \\ 
			\bm{\mu}_3^{'} &= \bm{B}_2^{-1} (\vect{\mu}_3 - \vect{\mu}_2) \\ 
			\bm{\Sigma}_3^{'} &= \bm{B}_2^{-1} \bm{\Sigma}_3 \bm{B}_2^\top \\ 
		\end{aligned} \right.
	\end{equation}
	Then by Proposition~\ref{proposition_DiffPreservesKLDivergence} we have:
	\begin{equation*}
		\left\{
		\begin{aligned}
			\KL\left(\N{N}_1^{'} \, \| \, \N{N}_2^{'} \right) &= \KL\left(\N{N}_1 \, \| \, \N{N}_2 \right) = \Delta_1 \\ 
			\KL\left ( \N{N}_2^{'} \,\|\, \N{N}_3^{'} \right) &= \KL\left(\N{N}_2 \,\|\, \N{N}_3\right) = \Delta_2
		\end{aligned}
		\right.
	\end{equation*}
	that is,
	\begin{equation*}
		\left\{
		\begin{aligned}
			\KL\left(\N{N}_1^{'} \, \| \, \N{N}(\vect{0},~\bm{I}) \right) = \Delta_1 \\
			\KL\left ( \N{N}(\vect{0},~\bm{I}) \,\|\, \N{N}_3^{'} \right) = \Delta_2
		\end{aligned}
		\right.
	\end{equation*}
	By Lemma~\ref{lem_KLSupremumForNormal} we obtain:
	\begin{equation} \label{eq_SupermumOfKL13}
		\KL\left(\N{N}_1 \,\|\, \N{N}_3\right) = \KL\left ( \N{N}_1^{'} \,\|\, \N{N}_3^{'} \right) \leq H^*(\Delta_1, \Delta_2)
	\end{equation}
	equality is achieved if and only if the parameters satisfy
	\begin{equation} \label{eq_mu'sigma'Q}
		\left\{
		\begin{aligned}
			\bm{\mu}_1^{'} &= \vect{0} \\
			\bm{\mu}_3^{'} &= \vect{0} \\
			\bm{\Sigma}_1^{'} &= \bm{Q} \operatorname{diag}\bigl(w_2(2\Delta_1), 1, \dots, 1\bigr) \bm{Q}^\top \\
			\bm{\Sigma}_3^{'} &= \bm{Q} \operatorname{diag}\bigl(w_2(2\Delta_2)^{-1}, 1, \dots, 1\bigr) \bm{Q}^\top
		\end{aligned}
		\right.
	\end{equation}
	Combining \eqref{eq_mu'sigma'B} and \eqref{eq_mu'sigma'Q}, we conclude that equality in \eqref{eq_SupermumOfKL13} is achieved if and only if the parameters satisfy 
	\begin{equation*}
		\left\{
		\begin{aligned}
			\bm{\mu}_1 &= \bm{\mu}_2 \\
			\bm{\mu}_3 &= \bm{\mu}_2 \\
			\bm{\Sigma}_1 &= \bm{B}_2 \bm{Q} \operatorname{diag}\bigl(w_2(2\Delta_1), 1, \dots, 1\bigr) \bm{Q}^\top \bm{B}_2^\top \\
			\bm{\Sigma}_3 &= \bm{B}_2 \bm{Q} \operatorname{diag}\bigl(w_2(2\Delta_2)^{-1}, 1, \dots, 1\bigr) \bm{Q}^\top \bm{B}_2^\top
		\end{aligned}
		\right.
	\end{equation*}
	where invertible matrix $\bm{B}_2$ satisfies $\bm{\Sigma}_2 = \bm{B}_2 \bm{B}_2^{\top}$, $\bm{Q}$ is an orthogonal matrix.
	The theorem is thus proved.
\end{IEEEproof}

\section{Proof of Theorem~\ref{theorem_KLUpperBoundForEpsilon}} \label{section_KLUpperBoundForEpsilon}

\begin{IEEEproof}
	When $\epsilon > 0$ are small constants, by Equation (J.191) in \cite{zhang2023PropertiesKullbackleiblerDivergence}, we have:
	\begin{equation*}
		w_2(\epsilon) = -W_{-1}\!\big(-e^{-(1+\epsilon)}\big) = 1 + \sqrt{2 \epsilon} + O(\epsilon)
	\end{equation*}
	Since $\epsilon_1, \epsilon_2 > 0$ are fixed small constants, then we have:
	\begin{equation*}
		\begin{aligned}
			F(\epsilon_1, \epsilon_2) &= \left[ w_2(\epsilon_1) - 1 \right] \left[ w_2(\epsilon_2) - 1 \right] + \epsilon_1 + \epsilon_2 \\
			&= (\sqrt{2 \epsilon_1} + O(\epsilon_1))(\sqrt{2 \epsilon_2} + O(\epsilon_2)) + \epsilon_1 + \epsilon_2 \\
			&= \epsilon_1 + \epsilon_2 + 2 \sqrt{\epsilon_1 \epsilon_2} + o(\epsilon_1) + o(\epsilon_2)
		\end{aligned} 
	\end{equation*}
	Therefore, by Theorem~\ref{theorem_KLSupremumForGeneral}, we have:
	\begin{equation*}
		\begin{aligned}
			\KL\left(\mathcal{N}_1 \, \| \, \mathcal{N}_3 \right) \leq& \frac{1}{2} F(2\epsilon_1, 2\epsilon_2)\\
			=& \frac{1}{2} \left[ 2 \epsilon_1 + 2 \epsilon_2 + 2 \sqrt{(2 \epsilon_1) (2 \epsilon_2)} + o(\epsilon_1) + o(\epsilon_2) \right] \\
			=& \epsilon_1 + \epsilon_2  + 2 \sqrt{\epsilon_1 \epsilon_2} + o(\epsilon_1) + o(\epsilon_2) 
		\end{aligned} 
	\end{equation*}
	The theorem is proved.
\end{IEEEproof}

\section{More Numerical Experiment} \label{appendix_HExperiment}
\noindent
The proof of Lemma~\ref{lem_MaxH} is conducted based on the properties of function $H(x,y;\Delta_1, \Delta_2)$.
We perform numerical experiments to visually demonstrate the surface of $H$.

Initially, for ease of comparison across different parameter pairs $(\Delta_1, \Delta_2)$, we map the original domain $\Omega(\Delta_1, \Delta_2)$ to the unit square $[0, 2]^{2}$ and normalize the range so that the maximum value becomes $\frac{1}{2}$. Specifically, we introduce the normalized bivariate function
\begin{equation*}
	\begin{aligned}
		\bar{H}(\bar{x}, \bar{y}; \Delta_1, \Delta_2) &= \frac{H(\bar{x} \Delta_1 , \bar{y} \Delta_2; \Delta_1, \Delta_2) }{F(2 \Delta_1, 2 \Delta_2)} \\
	\end{aligned} 
\end{equation*} 
with $(\bar{x}, \bar{y}) \in [0, 2]^{2}$. Then for any fixed $(\Delta_1, \Delta_2)$ and any $(x, y) \in \Omega(\Delta_1, \Delta_2)$, we have
\begin{equation*}
	\begin{aligned}
		H(x, y; \Delta_1, \Delta_2) = 
		\bar{H}\left(\frac{x}{\Delta_1}, \frac{y}{\Delta_2}; \Delta_1, \Delta_2\right) \cdot F(2 \Delta_1, 2 \Delta_2)\\
	\end{aligned} 
\end{equation*} 
and 
\begin{equation*}
	\begin{aligned}
		\left( x^*, y^* \right) &= \left( \bar{x}^* \Delta_1, \bar{y}^* \Delta_2 \right) \\
	\end{aligned}
\end{equation*}
where 
\begin{equation*}
	(\bar{x}^*, \bar{y}^*) = \arg \max_{0 \leq \bar{x}^{'} \leq 2, 0 \leq \bar{y}^{'} \leq 2} \bar{H}(\bar{x}, \bar{y}; \Delta_1, \Delta_2) 
\end{equation*}

Moreover, to investigate the solutions of the following equations
\begin{equation*}
	\left\lbrace 
	\begin{aligned}
		\frac{\partial \bar{H}(\bar{x}, \bar{y}; \Delta_1, \Delta_2) }{\partial \bar{x}} = 0 \\
		\frac{\partial \bar{H}(\bar{x}, \bar{y}; \Delta_1, \Delta_2) }{\partial \bar{y}} = 0
	\end{aligned}
	\right.
\end{equation*}
we further define the univariate functions
\begin{equation*}
	\left\lbrace 
	\begin{aligned}
		\bar{x}^*{(\bar{y})} &= \arg \max_{0 \leq \bar{x}^{'} \leq 2} \bar{H}(\bar{x}^{'}, \bar{y}; \Delta_1, \Delta_2) \\
		\bar{y}^*{(\bar{x})} &= \arg \max_{0 \leq \bar{y}^{'} \leq 2} \bar{H}(\bar{x}, \bar{y}^{'}; \Delta_1, \Delta_2)
	\end{aligned}
	\right.
\end{equation*}
such that $(\bar{x}^*{(\bar{y})}, \bar{y})$ satisfies $\frac{\partial \bar{H}(\bar{x}, \bar{y}; \Delta_1, \Delta_2) }{\partial \bar{x}} = 0$ and $(\bar{x}, \bar{y}^*{(\bar{x})})$ satisfies $\frac{\partial \bar{H}(\bar{x}, \bar{y}; \Delta_1, \Delta_2) }{\partial \bar{y}} = 0$.

Based on this formulation, we conduct numerical experiments on the function $\bar{H}(\bar{x}, \bar{y}; \Delta_1, \Delta_2)$ and trace the curves $(\bar{x}, \bar{y}^*{\left( \bar{x} \right)})$ and $(\bar{x}^*{(\bar{y})}, \bar{y})$ for various parameter pairs $(\Delta_1, \Delta_2)$. The results are presented in Figure~\ref{fig_H}. 
\begin{figure*}[ht]
	\vskip 0.2in
	\begin{center}
		\centering
		\includegraphics[width=0.99 \linewidth]{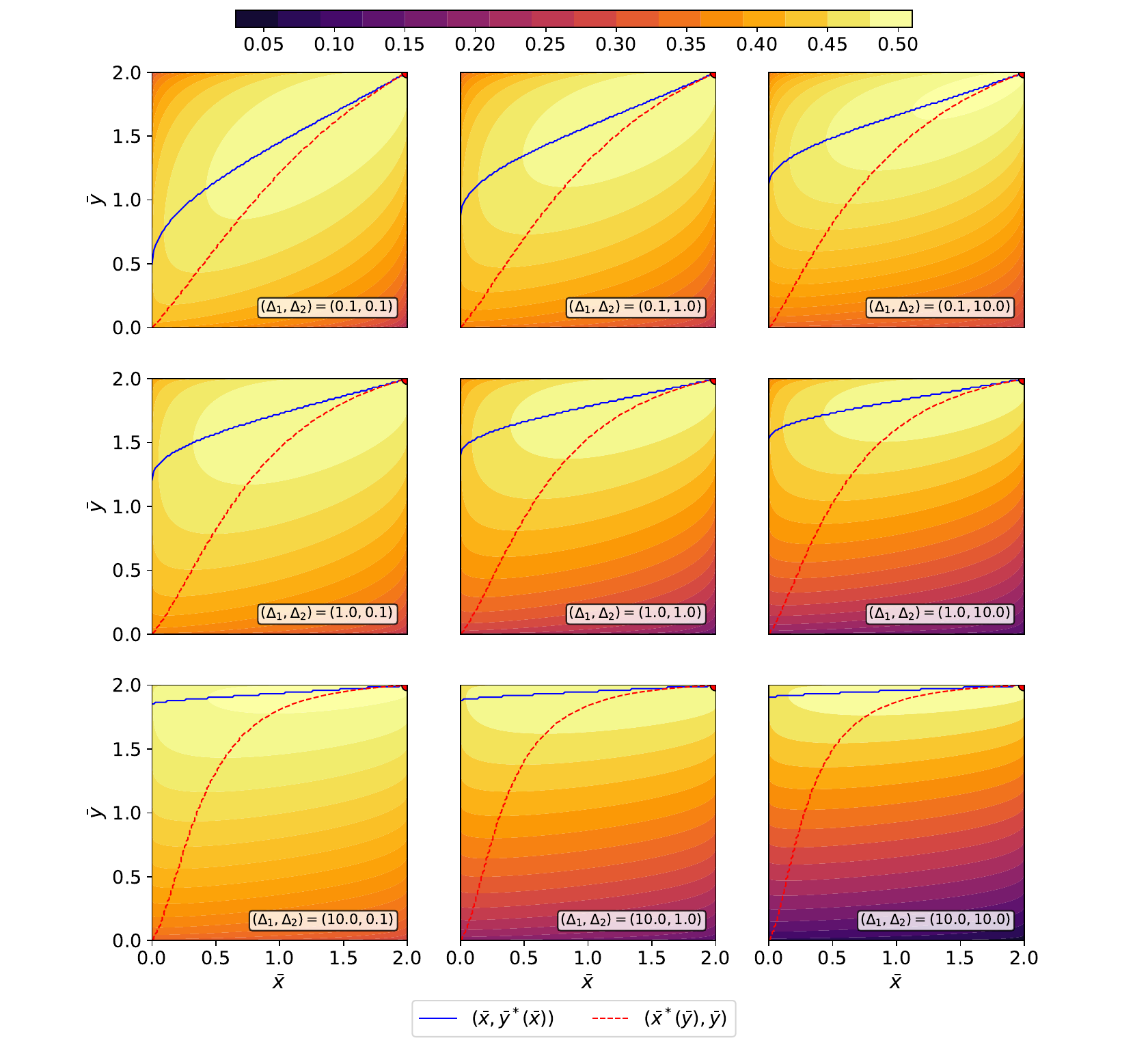}
		\caption{Heatmaps of $\bar{H}(\bar{x}, \bar{y}; \Delta_1, \Delta_2)$ over $(\bar{x}, \bar{y}) \in [0,2]^2$ for various parameter pairs $(\Delta_1, \Delta_2)$. 
			On each subplot, the blue curve $(\bar{x}, \bar{y}^*{\left( \bar{x} \right)})$ and the red curve $(\bar{x}^*{(\bar{y})}, \bar{y})$ trace the maximizers of $\bar{H}$ along vertical lines $\bar{x} = \text{const}$ and horizontal lines $\bar{y} = \text{const}$, respectively, satisfying $\frac{\partial \bar{H}(\bar{x}, \bar{y}; \Delta_1, \Delta_2)}{\partial \bar{y}} = 0$ and $\frac{\partial \bar{H}(\bar{x}, \bar{y}; \Delta_1, \Delta_2)}{\partial \bar{x}} = 0$, respectively. 
			The red point in the upper-right corner of each subplot denotes the global maximizer $(\bar{x}^*, \bar{y}^*)$.}
		\label{fig_H}
	\end{center}
\end{figure*}

As shown in Figure~\ref{fig_H}, all experiments consistently indicate that the curves $(\bar{x}, \bar{y}^*{\left( \bar{x} \right)})$ and $(\bar{x}^*{(\bar{y})}, \bar{y})$ have no intersection in $(0, 2)^{2}$, which implies that there is no critical point or maximizer in the interior of the domain and that the maximizer of $\bar{H}(\bar{x}, \bar{y}; \Delta_1, \Delta_2)$ is $(\bar{x}^*, \bar{y}^*) = (2, 2)$.
These observations imply that $(x^*, y^*) = (2\Delta_1, 2\Delta_2)$ and $H^*(\Delta_1, \Delta_2) = \frac{1}{2} F(2\Delta_1, 2\Delta_2)$, fully consistent with Lemma~\ref{lem_MaxH}.

\newpage

%
%
%
%
%

\end{document}